%% file: main.tex
\documentclass{article}

\input{math_commands.tex}

\usepackage{microtype}
\usepackage{graphicx}
\usepackage{subcaption}
\usepackage{booktabs} 
\usepackage{hyperref}

\usepackage[font=small]{caption}

\usepackage{bm}

\usepackage[accepted]{icml2023}

\usepackage{amsmath}
\usepackage{amssymb}
\usepackage{mathtools}
\usepackage{amsthm}
\usepackage{multirow}

\usepackage[capitalize,noabbrev]{cleveref}

\theoremstyle{plain}
\newtheorem{theorem}{Theorem}[section]

\theoremstyle{definition}

\newtheorem{assumption}[theorem]{Assumption}
\theoremstyle{remark}

\usepackage[textsize=tiny]{todonotes}

\icmltitlerunning{Fed-CBS: A Heterogeneity-Aware Client Sampling Mechanism for Federated Learning via Class-Imbalance Reduction}

\begin{document}

\twocolumn[
\icmltitle{Fed-CBS: A Heterogeneity-Aware Client Sampling Mechanism for Federated Learning via Class-Imbalance Reduction}




\begin{icmlauthorlist}
\icmlauthor{Jianyi Zhang}{duke}
\icmlauthor{Ang Li}{umd}
\icmlauthor{Minxue Tang}{duke}
\icmlauthor{Jingwei Sun}{duke}
\icmlauthor{Xiang Chen}{gmu}
\icmlauthor{Fan Zhang}{yale}
\icmlauthor{Changyou Chen}{ub}
\icmlauthor{Yiran Chen}{duke}
\icmlauthor{Hai Li}{duke}
\end{icmlauthorlist}

\icmlaffiliation{duke}{Duke University}
\icmlaffiliation{umd}{University of Maryland, College Park}
\icmlaffiliation{gmu}{George Mason University}
\icmlaffiliation{yale}{Yale University}
\icmlaffiliation{ub}{The State University of New York at Buffalo}

\icmlcorrespondingauthor{Jianyi Zhang}{jianyi.zhang@duke.edu}

\icmlkeywords{Machine Learning, ICML}

\vskip 0.3in
]



\printAffiliationsAndNotice{}  


\begin{abstract}
Due to the often limited communication bandwidth of edge devices, most existing federated learning (FL) methods randomly select only a subset of devices to participate in training at each communication round. 
Compared with engaging all the available clients, such a random-selection mechanism could lead to significant performance degradation on non-IID (independent and identically distributed) data.
In this paper, we present our key observation that the essential reason resulting in such performance degradation is the class-imbalance of the grouped data from randomly selected clients.  
Based on this observation, we design an efficient heterogeneity-aware client sampling mechanism, namely, Federated Class-balanced Sampling (Fed-CBS), which can effectively reduce class-imbalance of the grouped dataset from the intentionally selected clients.  We first propose a measure of class-imbalance which can be derived in a privacy-preserving way. 
Based on this measure, we design a computation-efficient client sampling strategy such that the actively selected clients will generate a more class-balanced grouped dataset with theoretical guarantees. 
Experimental results show that Fed-CBS outperforms the status quo approaches in terms of test accuracy and the rate of convergence while achieving comparable or even better performance than the ideal setting where all the available clients participate in the FL training.
\end{abstract}

\section{Introduction}
\label{sec:intro}
With the booming of IoT devices, a considerable amount of data is generated at the network edge, providing valuable resources for learning insightful information and enabling intelligent applications such as self-driving, video analytics, anomaly detection, etc. The traditional wisdom is to train machine learning models by collecting data from devices and performing centralized training. Data migration usually raises serious privacy concerns.  Federated learning (FL) \cite{mcmahan2017communication} is a promising technique to mitigate such privacy concerns, enabling a large number of clients to learn a shared model collaboratively, and the learning process is orchestrated by a central server. In particular, the participating clients first download a global model from the central server and then compute local model updates using their local data. The clients then transmit the local updates to the server, where the local updates are aggregated and then the global model is updated accordingly.

\begin{figure*}[h]
\centering
\begin{subfigure}{0.24\textwidth}
  \centering
  \includegraphics[width=0.9\linewidth]{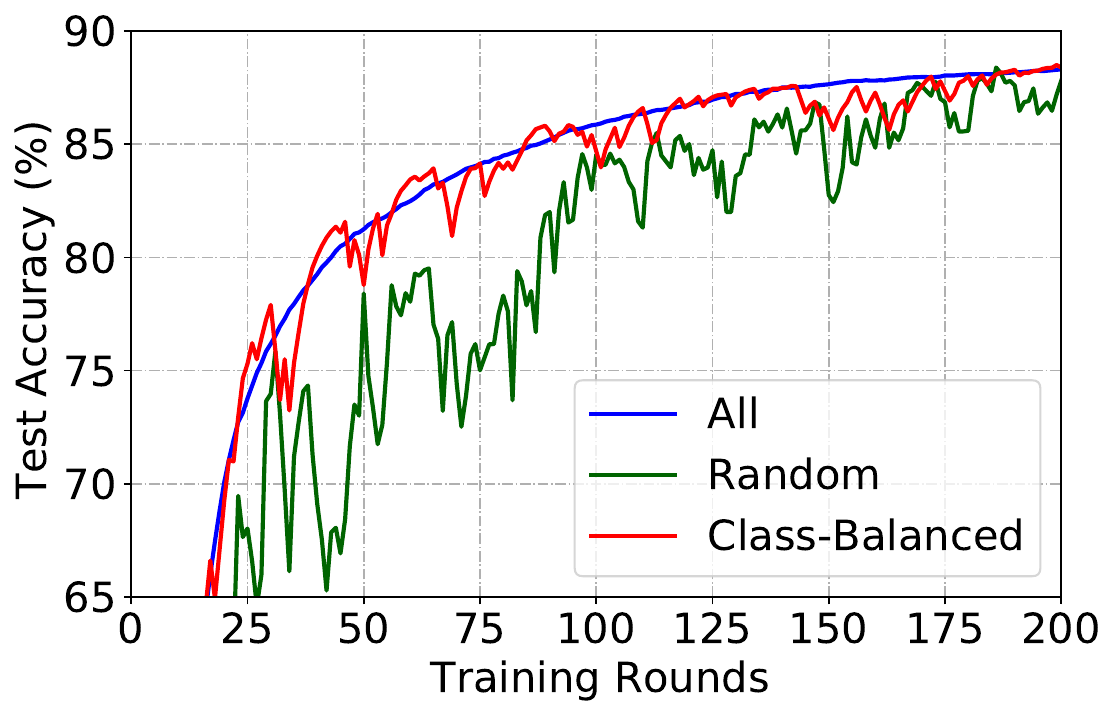}
  \caption{Global Balanced \& One-class}
  \label{fig:fig1-sub-first}
\end{subfigure}
\begin{subfigure}{.25\textwidth}
  \centering
  \includegraphics[width=0.9\linewidth]{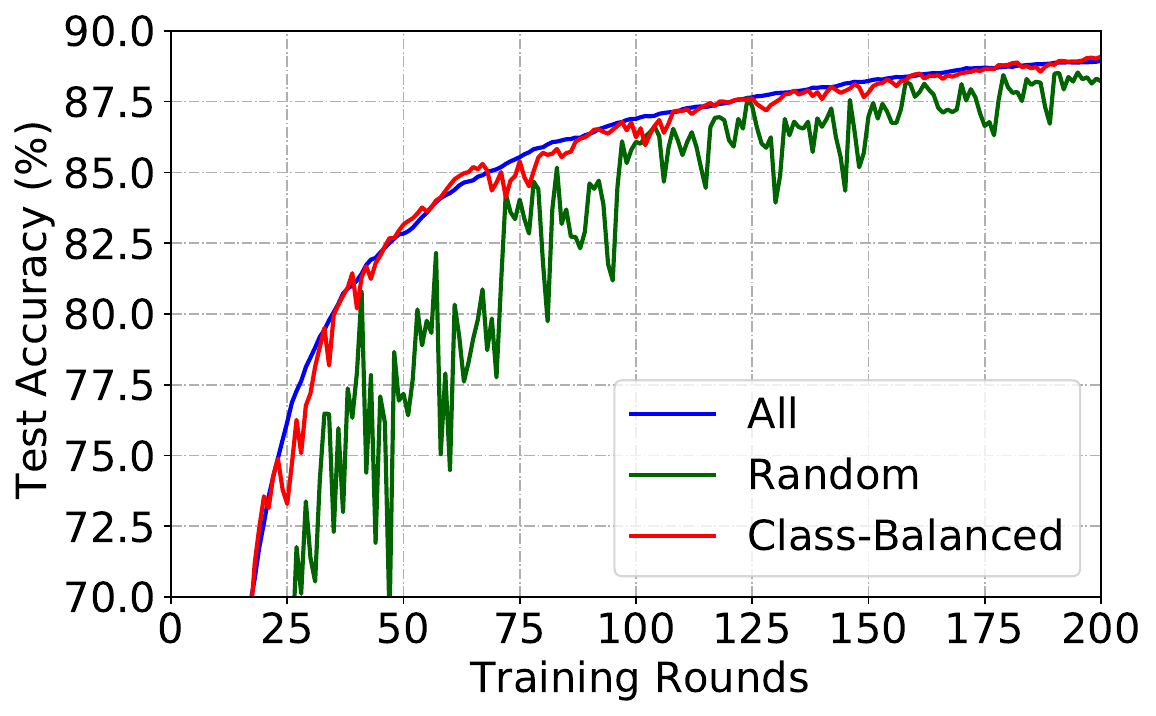}  
  \caption{Global Balanced \& Two-class}
  \label{fig:fig1-sub-second}
\end{subfigure}
\begin{subfigure}{.24\textwidth}
  \centering
  \includegraphics[width=0.9\linewidth]{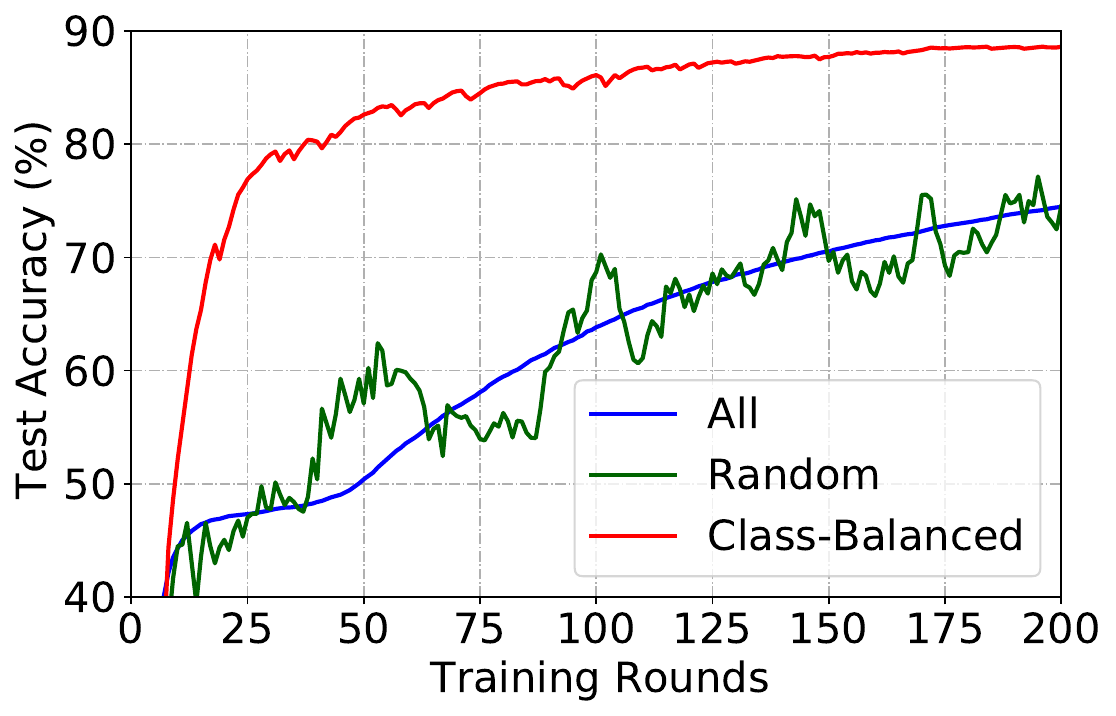}
  \caption{ Global Imbalanced \& One-class}
  \label{fig:fig1-sub-third}
\end{subfigure}
\begin{subfigure}{.24\textwidth}
  \centering
  \includegraphics[width=0.9\linewidth]{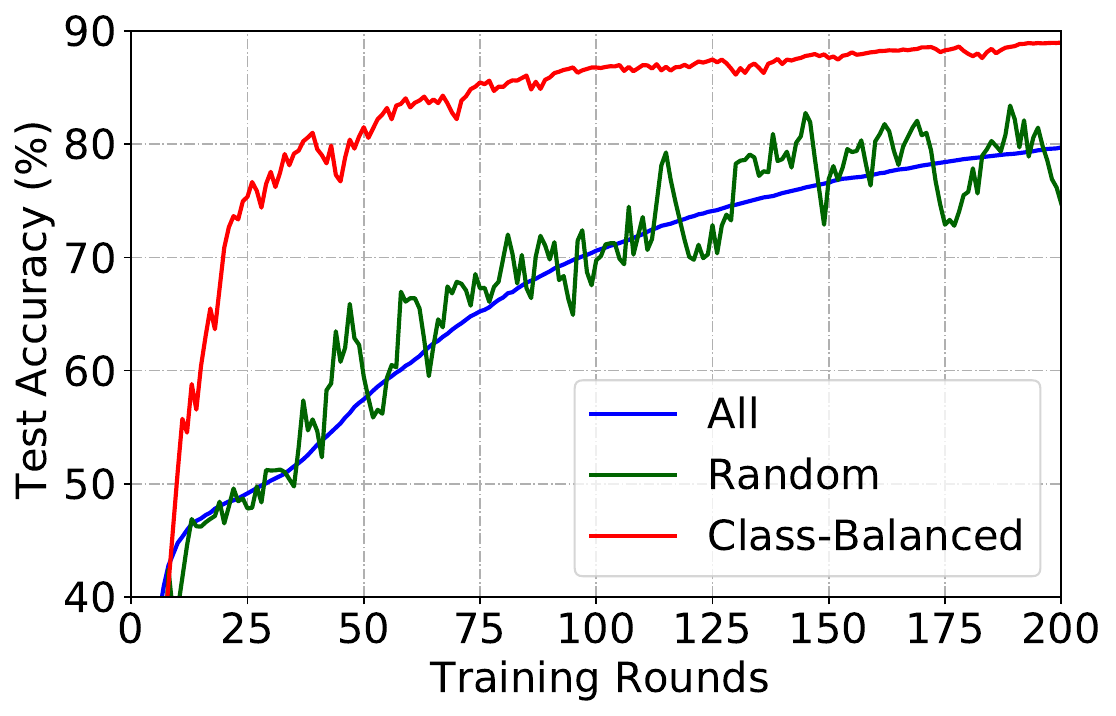}
  \caption{Global Imbalanced \& Two-class}
  \label{fig:fig1-sub-fourth}
\end{subfigure}
\caption{Three different FL client selection strategies on MNIST. \textit{All} means engaging all the 100 clients in training. \textit{Random} means randomly selecting 10 clients. \textit{Class Balanced} means that we keep the class-balance by intentionally selecting 10 clients. In Figure \ref{fig:fig1-sub-first} and \ref{fig:fig1-sub-second}, the global dataset of all the 100 clients' training data is class-balanced. In Figure \ref{fig:fig1-sub-third} and \ref{fig:fig1-sub-fourth}, the global dataset is class-imbalanced.
Each client has only one class of data in (a) and (c) and each client has two classes of data in (b) and (d). The results show significant performance degradation with imbalanced data from random client selection. It is worth noting that when the global dataset is class-imbalanced, selecting all the clients leads to worse performance compared with the \textit{Class Balanced} strategy, which suggests the importance of keeping class-balance for client selection.
}
\label{fig1}
\vspace{-10pt}
\end{figure*}

In practice, due to limited communication and computing capabilities, one usually can not engage all the available clients in FL training to fully utilize all the local data. Therefore, most FL methods only randomly select a subset of the available clients to participate in the training in each communication round. However, in practice, the data held by different clients are often typically non-IID (independent and identically distributed) due to various user preferences and usage patterns. This leads to a serious problem that the random client selection strategy often fails to learn a global model that can generalize well for most of the participating clients under non-IID settings~\cite{goetz2019active,Cho2020ClientSI,Nishio2019ClientSF,yang2020federated}. 

Several heuristic client selection mechanisms have been proposed to tackle the non-IID challenge. For example, in the method of ~\cite{goetz2019active}, the clients with larger local loss will have a higher  probability to be selected to participate in the training. Power-of-Choice~\cite{Cho2020ClientSI} selects several clients with the largest loss from a randomly sampled subset of all the available clients. However, selecting clients with a larger local loss may not guarantee that the final model can have a smaller global loss. Another limitation of previous research on client selection is the missing comparison between their strategy and the ideal case, where all the available clients participate in the training. In general, existing works not only miss a vital criterion that can measure the performance of their methods, but also fail to investigate the essential reason why random client selection can lead to performance degradation on non-IID data compared with fully engaging all the available clients. 

In this paper, we focus on image classification tasks. First, we demonstrate our key observation for the essential reason why random client selection results in performance degradation on non-IID data, which is the \textit{class-imbalance} of the grouped dataset from randomly selected clients. Based on our observation, we design an efficient heterogeneity-aware client sampling mechanism, \textit{i.e.}, Federated Class-Balanced Sampling (Fed-CBS), which effectively reduces the class-imbalance in FL. Fed-CBS is orthogonal to numerous existing techniques to improve the performance of  FL~\cite{li2018federated,wang2020addressing,karimireddy2019scaffold,chen2020fedmax,2020arXiv200300295R,Hao_2021_CVPR,yang2021flop} on non-IID data, meaning Fed-CBS can be integrated with these methods to improve their performance further. 
Our major contributions are summarized as follows:
\vspace{-3mm}
\begin{itemize}
    \item We reveal that  the class-imbalance is the fundamental reason why random client selection leads to performance degradation on non-IID data in Section \ref{Sec:Preliminary}. \vspace{-2mm} 
    \item To effectively reduce the class-imbalance, we design an efficient heterogeneity-aware client sampling mechanism, \textit{i.e.}, Fed-CBS, based on our proposed class-imbalance metric in Section \ref{Sec:Method}. We provide theoretical analysis on the convergence of Fed-CBS  in Section \ref{Sec:Convergence}, as well as the analysis of the NP-hardness of this problem. \vspace{-2mm}
    \item We empirically evaluate Fed-CBS on FL benchmark (non-IID datasets) in Section \ref{Sec:experiments}. The results demonstrate that Fed-CBS can improve the accuracy of FL models on CIFAR-10 by $2\%\sim 7\%$ and accelerate the convergence time by $1.3\times\sim 2.8\times$, compared with the state-of-the-art method ~\cite{yang2020federated} that also aims to  reduce class-imbalance via client selection. Furthermore, our Fed-CBS achieves comparable or even better performance than the ideal setting where all the available devices are involved in the training. 
\end{itemize}

\section{Preliminary and Related Work}\label{Sec:Preliminary}
We first clarify three definitions. The \textbf{local dataset} is the client's own locally-stored dataset, which is inaccessible to other clients and the server. Due to the heterogeneity of local data distribution, the phenomenon of class-imbalance frequently happens in most of the local datasets. The \textbf{global dataset} is the union of all the available client local datasets. It can be class-balanced or class-imbalanced, but it is often imbalanced. The \textbf{grouped dataset} is the union of several clients' local datasets which have been selected to participate in training for one communication round. It follows that the grouped dataset is a subset of the global dataset. 

\subsection{Pitfall of Class-Imbalance in Client Selection}\label{CImotivation}

 Some recent works~\cite{yang2020federated,wang2020addressing,Duan2019AstraeaSF} have identified the issue of class-imbalance in the grouped dataset by random selection under non-IID settings. Since class-imbalance degrades the classification accuracy
on minority classes \cite{huang2016learning} and leads to low training efficiency, we are motivated to verify whether the class-imbalance of the randomly-selected grouped dataset is the essential reason accounting for the performance degradation. 

We conduct some experiments on MNIST to verify our proposition\footnote{Detailed experiment settings are listed in the Appendix (Section \ref{setting2.1})}. As shown in Figure \ref{fig:fig1-sub-first} and Figure \ref{fig:fig1-sub-second}, the random selection mechanism shows the worst performance when the global label distribution is class-balanced. If we keep the grouped dataset class-balanced by manually selecting the clients based on their local label distribution, we can obtain accuracy comparable to the case of fully engaging all the clients in training.  



Another natural corollary is that when the global dataset is inherently class-imbalanced, engaging all clients in training may lead to worse performance than manually keeping the grouped dataset class-balanced. The results in Figure \ref{fig:fig1-sub-third} and Figure \ref{fig:fig1-sub-fourth} prove our hypothesis and verify the importance of class-imbalance reduction. This also indicates that only keeping diversity in the data and fairness for clients is not enough, which was missed in the previous literature \cite{balakrishnan2021,efficiencyBCSS,yang2020federated,wang2020addressing,shen2022agnostic,wang2021addressing}. More experimental results on larger datasets will be provided to verify the importance of class-imbalance reduction (Section \ref{Sec:experiments}).

\begin{figure*}[!htbp]
    \centering
    \includegraphics[width=\linewidth]{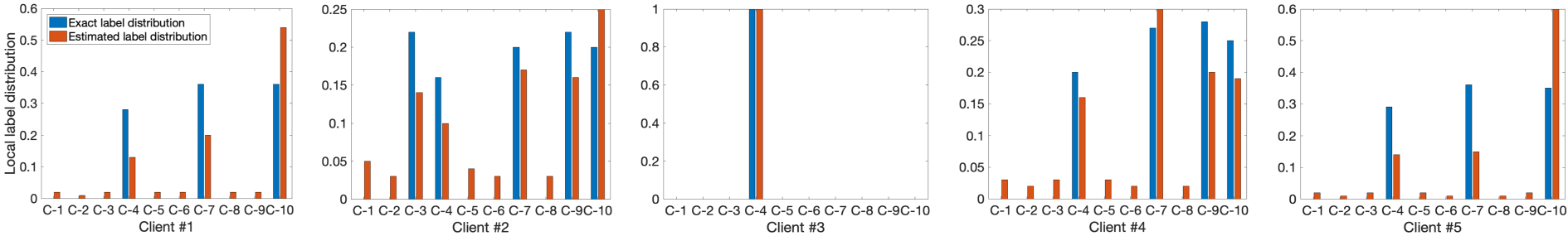}\\
    \vspace{-10pt}
    \caption{The exact local label distributions and the estimated ones of the first 5 clients in the experiment of \cite{yang2020federated}. Label distribution quantifies the ratio between the number of data from 10 classes (C-1, C-2, ..., C-10) in each client's local dataset.}
    \label{RealEstimated}
\end{figure*}

\begin{figure*}[!htbp]
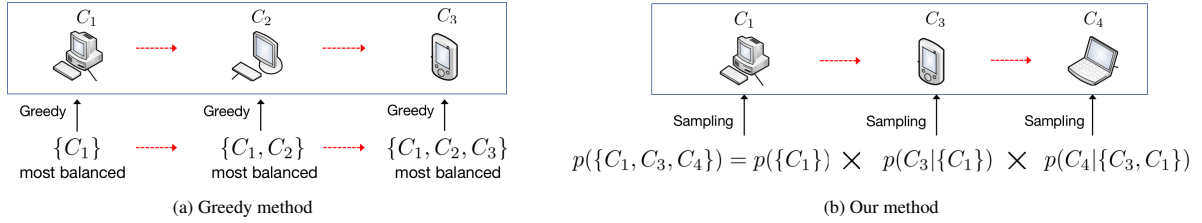

\begin{subfigure}{.49\textwidth}
  \centering
  \includegraphics[width=\linewidth]{Fig/greedy.pdf}
  \caption{Greedy method}
  \label{fig:greedy}
\end{subfigure}
\begin{subfigure}{.49\textwidth}
  \centering
  \includegraphics[width=0.99\linewidth]{Fig/ours.pdf}  
  \caption{Our method}
  \label{fig:our method}
\end{subfigure}\\
\vspace{-10pt}
\caption{An example demonstrating the weakness of greedy method to deal with class imbalance. Supposing we work on a 6-class classification task and aim to select 3 clients from 4 available clients $C_1, C_2, C_3, C_4$. Each of them has 30 images. The compositions of their local datasets are $[5,5,5,5,5,5]$, $[6,6,6,6,6,0]$, $[0,0,0,10,10,10]$ and  $[10,10,10,0,0,0]$ respectively. The greedy method in \cite{yang2020federated} is deterministic. It can only derive one result $\{C_1, C_2, C_3\}$ instead of the optimal solution $\{C_1, C_3, C_4\}$ (see the text description). But our method is based on probability modeling, which directly models the distribution of the optimal solution $\{C_1, C_3, C_4\}$. Thus when sampling from it, the optimal solution can be returned with high probability.}
\label{fig:greedyandours}
\vspace{-5pt}
\end{figure*}

\subsection{Related Work}\label{sec:relatedwork}
Some effort has been made to improve client selection for FL in previous literature. \citep{Cho2020ClientSI,goetz2019active} select clients with larger local loss, but this cannot guarantee that the final global model has a smaller global loss. Focusing on the diversity in client selection, the authors of \cite{balakrishnan2021} select clients by maximizing a submodular facility location function defined over gradient space. A fairness-guaranteed algorithm termed RBCS-F was proposed in \cite{efficiencyBCSS}, which models the fairness-guaranteed client selection as a Lyapunov optimization problem. Although diversity and fairness are important, the experimental results in Section \ref{CImotivation} demonstrate that they are not enough for client selection if the class-imbalance issue is not considered. The authors in \cite{ribero2020communication} model the progression of model weights by an Ornstein-Uhlenbeck process and design a
sampling strategy for selecting clients with significant weight
updates. However, the work only considers the identical data distribution setting.
Following the existing works~\cite{goetz2019active,Cho2020ClientSI}, we only focus on the data heterogeneity caused by non-IID data across clients. Additionally, we included a comparison of our method with other clustered-based client sampling algorithms in the appendix.


To the best of our knowledge, \cite{Duan2019AstraeaSF} and  \cite{yang2020federated} are the first two attempts to improve client selection by reducing class-imbalance. An extra virtual component called a mediator is introduced in \textit{Astraea} of \cite{Duan2019AstraeaSF}, which has access to the local label distributions of the clients. With these distributions, \textit{Astraea} will conduct client selection in a greedy way. The method of \cite{yang2020federated} first estimates the local label distribution of each client based on the gradient of model parameters and adopts the same greedy way to select clients as \textit{Astraea}. Since directly knowing the exact value of local label distributions of clients in \textit{Astraea} will cause severe concerns on privacy leakage, we consider the method in \cite{yang2020federated} as the state-of-the-art method aiming to improve client selection through class-imbalance reduction. 

However, the solution presented by \cite{yang2020federated} has several limitations. First, their method requires a class-balanced auxiliary dataset that consists of all classes of data at the server. However, that is not always available in some large-scale FL systems since it requires the server to collect raw data from clients, which breaches privacy. Second, their estimations of the clients' local label distribution are not accurate as shown in 
Figure \ref{RealEstimated}. Theorem 1 in \cite{yang2020federated} supports their estimations, but it cannot be generalized to multi-class classification tasks since it has only been proved in the original paper \cite{anand1993improved} for two-class classification problems. Finally, the performance of greedily conducting the client selection is not guaranteed due to the nature of the greedy algorithm. We provide an example in Figure \ref{fig:greedyandours} to show its weakness. Their method will select $C_1$ as the first client since it is the most class-balanced one. Then $C_2$ will be selected because the grouped dataset of $C_1\cup C_2$ is the most class-balanced among the choices $C_1\cup C_2$, $C_1\cup C_3$ and $C_1\cup C_4$. Similarly, it will choose $C_3$ since the grouped dataset of $C_1\cup C_2\cup C_3$ is more class-balanced than  $C_1\cup C_2\cup C_4$. Their method is deterministic and thus only one combination $\{C_1,C_2, C_3\}$ is obtained. However, this is clearly not the optimal solution since $\{C_1,C_3, C_4\}$ is more class-balanced than $\{C_1,C_2, C_3\}$. The above weaknesses motivate us to design a more effective solution for this problem. 
\vspace{-5pt}
\section{Methodology} \label{Sec:Method}
We first propose a metric to measure class-imbalance in Section \ref{sub:CIM}. Then we derive the measure with privacy-preserving techniques in Section \ref{subsec:FHE}. Based on this measure, we then design our client sampling mechanism and show its superiority in Section \ref{sec:samplingSt}.
\vspace{-10pt}
\subsection{Class-Imbalance Measure}\label{sub:CIM}
Assume there are $B$ classes of data in an image classification task, where $B\geq 2$. In the $k$-th communication round, we assume there are $N_k$ available clients and we select $M$ clients from them. To make the presentation concise, we ignore the index ``k'' and assume the set of indices for the available clients is  $\{1,2, 3, ..., N\}$ and the $n$-th available client has its own training dataset $\mathcal{D}_n$. We adopt the following vector of size $B$ to represent the local label distribution of $\mathcal{D}_n$, where $\alpha_{(n,b)}\geq 0$ and $\sum_{b=1}^B\alpha_{(n,b)}=1$,
\begin{align}
    \bm{\alpha}_n=\left[\alpha_{(n,1)}, \alpha_{(n,2)},...,\alpha_{(n,b)},... ,\alpha_{(n,B)}\right]~.
    \label{localdistribution}
\end{align}

We aim to find a subset $\mathcal{M}$ of $\{1,2,3,..,N\}$ of size $M$, such that the following grouped dataset ${\textstyle \mathcal{D}^{g}_{\mathcal{M}}=\bigcup\limits_{n\in \mathcal{M}} \mathcal{D}_n}$ is class-balanced. Assuming the $n$-th client's local dataset has $q_n$ training samples, the following vector $\bm{\alpha}^g_{\mathcal{M}}$  can represent the label distribution of the grouped dataset $\mathcal{D}^{g}_{\mathcal{M}}$, 
\vspace{-1mm}
{\small
\begin{align*}\label{groupdistribution}
      &\bm{\alpha}^g_{\mathcal{M}}=\frac{\sum_{n\in \mathcal{M} }q_n\bm{\alpha}_n}{\sum_{n\in \mathcal{M} }q_n}=\\
      &\left[\frac{\sum_{n\in \mathcal{M} }q_n \alpha_{(n,1)}}{\sum_{n\in \mathcal{M} } q_n},...,\frac{\sum_{n\in \mathcal{M} }q_n \alpha_{(n,b)}}{\sum_{n\in \mathcal{M} } q_n},... ,\frac{\sum_{n\in \mathcal{M} }q_n \alpha_{(n,B)}}{\sum_{n\in \mathcal{M} } q_n}\right]. 
\end{align*}}
Instead of dealing with the Kullback-Leibler (KL) divergence as \cite{Duan2019AstraeaSF,yang2020federated}, which is complicated to analyze, we propose the following function to measure the magnitude of class-imbalance of $\mathcal{M}$, which we call \textit{Quadratic Class-Imbalance Degree (QCID)}:{\begin{align*}
\textit{QCID} (\mathcal{M}) \triangleq \sum_{b=1}^B(\frac{\sum_{n\in \mathcal{M} }q_n \alpha_{(n,b)}}{\sum_{n\in \mathcal{M} } q_n}-\frac{1}{B})^2.
\end{align*}}Essentially, \textit{QCID}($\mathcal{M}$) reflects the $L_2$ distance between the distribution of the grouped dataset $\mathcal{D}^{g}_{\mathcal{M}}$ and the ideally class-balanced dataset that has a uniform label distribution. Although there exist several more commonly-used probabilistic distances other than $L_2$, it is easier to analyze \textit{QCID} and more efficient to calculate while keeping privacy as shown in the next section.

\subsection{Privacy-Preserving QCID Derivation}\label{subsec:FHE}

Our privacy goal is to calculate the value of \textit{QCID} while keeping clients' local distributions $\{\bm{\alpha}_n\}$ hidden from the server since it contains sensitive information. Unlike Kullback-Leibler (KL) divergence which is difficult to analyze, we can expand the expression of \textit{QCID} to explore how the pairwise relationships of the clients' local label distributions $\{\bm{\alpha}_m\}$ affects the class-imbalance degree of $\mathcal{M}$, where $m\in \mathcal{M}$. Below we provide a theorem to show the feasibility of our method.

\begin{theorem}
The \textit{QCID} value is decided by the sum of inner products between each two vectors $\bm{\alpha}_m,\bm{\alpha}_{m^\prime} \in \{\bm{\alpha}_m\}$ with $m\in \mathcal{M}$, {\it i.e.},
\begin{align*}
    \textit{QCID} (\mathcal{M})=\frac{\sum_{n\in \mathcal{M},n^\prime \in \mathcal{M}} q_nq_n^\prime\bm{\alpha}_n\bm{\alpha}_{n^\prime}^T}{(\sum_{n\in \mathcal{M} }q_n)^2} -\frac{1}{B}
\end{align*}\label{theorem1}
\end{theorem}
\vspace{-15pt}

Theorem \ref{theorem1} reveals the fact that there is no need to know the local label distribution of each client to calculate the \textit{QCID}, as long as we have access to the inner products between each other. To derive the \textit{QCID} for any subset $\mathcal{M} \subseteq \{1,2,3,..,N\}$, we only need to know the following $N\times N$ matrix $\bm{S}$ with element $s_{n,n^\prime}$ being $\bm{\alpha}_n\bm{\alpha}_{n^\prime}^T$, which is the inner product between the local label distributions of the available clients $n$ and $n^\prime$. 
\begin{align*}
 \bm{S}=
 \begin{bmatrix}
q_1q_1\bm{\alpha}_1\bm{\alpha}_{1}^T & q_1q_2\bm{\alpha}_1\bm{\alpha}_{2}^T  & \cdots   & q_1q_N\bm{\alpha}_1\bm{\alpha}_{N}^T   \\ 
q_2q_1\bm{\alpha}_2\bm{\alpha}_{1}^T & q_2q_2\bm{\alpha}_2\bm{\alpha}_{2}^T  & \cdots   & q_2q_N\bm{\alpha}_2\bm{\alpha}_{N}^T  \\
\vdots & \vdots  & \ddots   & \vdots  \\
q_Nq_1\bm{\alpha}_{N}\bm{\alpha}_{1}^T & q_Nq_2\bm{\alpha}_{N}\bm{\alpha}_{2}^T  & \cdots\  & q_Nq_N\bm{\alpha}_{N}\bm{\alpha}_{N}^T  \\
\end{bmatrix}
\end{align*}
Although it is possible to calculate  \textit{QCID} with $\bm{S}$, another concern arises, \textit{can a malicious party infer the values of $\{\bm{\alpha_i}\}$ from  $\mS$ ?} Then we have another theorem to provide privacy protection.
\begin{theorem}\label{theoremprivacy}
One can not derive the values of $\{\bm{\alpha_i}\}$ from the value of $\mS$.
\end{theorem}
\vspace{-2mm}
Based on these two theorems, our privacy goal can be simplified as enabling the server to derive $\bm{S}$ without access to $\{\bm{\alpha_i}\}$. There are several ways to achieve our goal. One option is to leverage the server-side trusted execution environments (TEEs), e.g., Intel SGX~\cite{anati2013innovative}, which allows calculating $S$ without leaking information of $\{\alpha_n\}$. Another potential solution is to adopt Fully Homomorphic Encryption (FHE)~\cite{seal,bgvfhe,fvfhe,helib,halevi2015bootstrapping} to enable the server to compute on encrypted data (\textit{i.e.}, $\{\bm{\alpha_i}\}$) to derive $\mS$.  We provide an example of the system skelon in Section \ref{FHEexample} to illustrate how to derive ${S}$ without knowing the local label distributions $\{\bm{\alpha_i}\}$ using FHE.  Since we focus on efficient algorithms to reduce class-imbalance instead of designing the fundamental infrastructure for computing (which is beyond our scope and not a contribution of this paper), we leave the detailed system design for future work.

\subsection{A Client Sampling Mechanism}\label{sec:samplingSt}
To select the most class-balanced grouped dataset $\mathcal{D}^g_{\mathcal{M}}$, we need to find the optimal subset $\mathcal{M}^*$ that has the lowest $QCID$ value, which is defined as follows:  
{\begin{align*}
    \mathcal{M}^*\triangleq \mathop{\arg\min}_{\mathcal{M} \subseteq \{1,2,3,..,N\} }\frac{\sum_{n\in \mathcal{M},n^\prime \in \mathcal{M}} q_nq_{n^\prime}\bm{\alpha}_n\bm{\alpha}_{n^\prime}^T}{(\sum_{n\in \mathcal{M} }q_n)^2} -\frac{1}{B}.
\end{align*}}
The main challenge is computational complexity. To find the exact optimal $\mathcal{M}^*$, we need to loop through all the possible cases and find the lowest $QCID$ value. The computational complexity thereafter will be $\mathcal{O}\left(\tbinom{N}{M}\times M^2\right)$, which is unacceptable when $N$ is extremely large.
\vspace{-5pt}
\paragraph{A probability approach}
To overcome the computational bottleneck, instead of treating $\mathcal{M}$ as a determined set, we consider it as a sequence of random variables, \textit{i.e.} $\mathcal{M}=\{C_1,C_2,...,C_m,...,C_M\}$ and assign it with some probability. Our expectation is that $\mathcal{M}$ should have higher probability to be sampled with if it is more class-balanced. This means $P(C_1=c_1,C_2=c_2,...,C_m=c_m,...,C_M=c_M)$ should be larger if $\mathcal{M}=\{c_1,c_2,...,c_M\}$ has a lower $QCID$ value.
Our sampling strategy generates the elements in $\mathcal{M}$ in a sequential manner, {\it i.e.}, we first sample $\mathcal{M}_1=\{c_1\}$ according to the probability of
$P(C_1=c_1)$, then sample $c_2$ to form $\mathcal{M}_2=\{c_1,c_2\}$ according to the conditional probability $P(C_2=c_2|C_1=c_1)$. The same procedure applies for the following clients 
until we finally obtain $\mathcal{M}=\{c_1,c_2,...,c_M\}$. In the following, we will design proper conditional probabilities such that the joint distribution of client selection satisfies our expectations.

Let $T_{n}$ denote the number of times that client $n$ has been selected. Once client $n$ has been selected in a communication round, $T_{n} \rightarrow T_{n}+$ 1, otherwise, $T_{n} \rightarrow T_{n}$. Inspired by combinatorial upper confidence bounds (CUCB) algorithm \cite{pmlr-v28-chen13a} and previous work in \cite{yang2020federated}, in the $k$-th communication round, the first element is designed to be sampled with the following probability:
\vspace{-5pt}
\begin{align*}
    P(C_1=c_1) \propto \frac{1}{[QCID(\mathcal{M}_1)]^{\beta_1}}+\lambda \sqrt{\frac{3 \ln k}{2 T_{c_1}}}, \quad \beta_1>0,~
\end{align*}
where $\lambda$ above is the exploration factor to balance the trade-off between exploitation and exploration. The second term will add a higher probability to the clients that have never been sampled before in the following communication rounds. After sampling $C_1$, the second client is defined to be sampled with probability
{ \begin{align*}
    P(C_2=c_2|C_1=c_1) \propto \frac{\frac{1}{[QCID(\mathcal{M}_2)]^{\beta_2}}}{\frac{1}{[QCID(\mathcal{M}_1)]^{\beta_1}}+\alpha \sqrt{\frac{3 \ln k}{2 T_{c_1}}}}, \beta_2>0.
\end{align*}}
For the m-th client, where $2 < m \leq M$, we define 
{ \begin{align*}
    P(C_m=c_m|C_{m-1}=c_{m-1},...,C_2=c_2,C_1=c_1) \\
    \propto  \frac{[QCID(\mathcal{M}_{m-1})]^{\beta_{m-1}}}{[QCID(\mathcal{M}_m)]^{\beta_m }}, \quad \beta_{m-1},\beta_{m}>0.
\end{align*}}
With the above sampling process, the final probability to sample $\mathcal{M}$ is $P(C_1=c_1,C_2=c_1,...,C_M=c_M)=P(C_1=c_1)\times P(C_2=c_2|C_1=c_1)\cdots\times P(C_M=c_M|C_{M-1}=c_{M-1},...,C_2=c_2,C_1=c_1)\propto {1}/{[QCID(\mathcal{M})]^{\beta_M }}$. Since $\beta_M>0$, this matches our goal that the $\mathcal{M}$ with lower $QCID$ value should have higher probability to be sampled with. Our mechanism, Fed-CBS, is summarized in Algorithm \ref{Algorithm123}  .
\begin{algorithm}[!htbp]
\begin{algorithmic}
\caption{Fed-CBS}\label{Algorithm123}
\STATE {\bfseries Initialization:} initial local model $\boldsymbol{w}^{(0)}$, client index subset $\mathcal{M}=\varnothing$, $K$ communication rounds, $k=0$, $T_n=1$\\
\WHILE{$k < K$}
\STATE \textbf{\color{red}// Client Selction}:
\FOR{$n$ {\bfseries in} $\{1,2,...,N\}$}
\IF{$n\in\mathcal{M}$}
   \STATE $T_n\rightarrow T_n+1$
   \ELSE 
   \STATE $T_n\rightarrow T_n$;
   \ENDIF
\ENDFOR
\STATE Update $\mathcal{M}$ using our proposed sampling strategy in Section \ref{sec:samplingSt}\\
\STATE \textbf{\color{red}// Local Updates:}
\FOR{$n\in \mathcal{M}$}
\STATE $\boldsymbol{w}_n^{(k+1)}\leftarrow \ Update(\boldsymbol{w}^{(k)})$. \\
\ENDFOR
\STATE \textbf{\color{red}// Global Aggregation}:
\STATE $\boldsymbol{w}^{(k+1)}\leftarrow Aggregate(\boldsymbol{w}_n^{(k+1)})$ for $n \in\mathcal{M}$
\ENDWHILE
\end{algorithmic}    
\end{algorithm}

\paragraph{Details and analysis}\label{remark1}
For any $1<m < M$, we have 
{\begin{align*}
    P(C_1=c_1,C_2=c_1,...,C_m=c_m)\propto \frac{1}{[QCID(\mathcal{M}_m)]^{\beta_m }}.
\end{align*}}
This means when we generate the first $m$ elements of $\mathcal{M}$, we expect the $\mathcal{M}_{m}$ should be more class-balanced since the $\mathcal{M}_{m}$ with lower $QCID$ value has a higher probability of being sampled. This is different from the algorithm in \cite{yang2020federated}, which greedily chooses the $c_m$ from $\{1,2,..,N\}/\mathcal{M}_{m-1}$ that makes $\mathcal{M}_{m}$ the most class-balanced one. 
Unlike the greedy algorithm which has no guarantees on finding the optimal client set, our method can generate the globally optimal set of clients in the sense of probability. 
An example is provided in Figure \ref{fig:greedyandours} to demonstrate that our method can overcome the pitfall of the greedy method. After selecting the first two clients, $\{C_1, C_3\}$ our method is less class-balanced than $\{C_1, C_2\}$ chosen by the greedy method. However, after making the last choice, our method has the chance to derive a perfectly class-balanced set $\{C_1, C_3, C_4\}$. In contrast, the greedy method can only get one result $\{C_1, C_2, C_3\}$, which is less class-balanced.

We require the distribution of $P(C_1=c_1,C_2=c_1,...,C_m=c_m)$ to be more dispersed when $m$ is small. This is because we expect our sampling strategy to explore more possible cases of client composition at the beginning. We require the distribution of $P(C_1=c_1,C_2=c_1,...,C_m=c_m)$ to be less dispersed when $m$ is large. This is because as we approach the end of our sampling process, we expect our sampling strategy can find the $\mathcal{M}_m$ that is more class-balanced. Especially when $m=M$, we hope the strategy to find the client $c_M$ which can make $\mathcal{M}$ the most class-balanced. Since 
\vspace{-0.1cm}
{ \begin{align*}
    &P(C_1=c_1,C_2=c_1,...,C_m=c_m)\propto \frac{1}{[QCID(\mathcal{M}_m)]^{\beta_m }}
\end{align*}
we can set $0<\beta_1 < \beta_2<...<\beta_M$ to satisfy the above requirements.}

{\em Remark: We set a lower bound for $QCID(\mathcal{M}_m)$ as $L_{b}$ since $QCID(\mathcal{M}_m)=0$ in some special cases will cause  $P(C_m=c_m|C_{m-1}=c_{m-1},...,C_1=c_1)\rightarrow \infty$.}
When viewing the conditional distribution as the likelihood in Bayesian inference, our probability can be interpreted as an estimate of the posterior distribution. This allows us to comprehend our algorithm through the lens of Bayesian sampling \cite{welling2011bayesian,liu2019stein,pmlr-v108-zhang20d,zhang2019cyclical}. In our future studies, we will further analyze the connection between them.
Below we present two theorems to show the superiority of our proposed sampling strategy. 
\begin{theorem} [Class-Imbalance Reduction]\label{CIR}
We denote the probability of selecting $\mathcal{M}$ with our strategy with $\beta_M$ as $P_{\beta_M}$ and the probability of selecting $\mathcal{M}$ with the random selection as $P_{rand}$. Our method can reduce the expectation of $QCID$ compared to the random selection mechanism. In other words, we have 
\begin{align*}
    \mathbb{E}_{\mathcal{M}\sim P_{\beta_M}} QCID(\mathcal{M}) < \mathbb{E}_{\mathcal{M}\sim P_{rand}} QCID(\mathcal{M}).
\end{align*}
Furthermore, if increasing the value $\beta_M$, the expectation of $QCID$ can be further reduced, {\it i.e.}, for $\beta^\prime_M>\beta_M$, we have
\begin{align*}
     \mathbb{E}_{\mathcal{M}\sim P_{\beta^\prime_M}} QCID(\mathcal{M})<\mathbb{E}_{\mathcal{M}\sim P_{\beta_M}} QCID(\mathcal{M}). 
\end{align*}
\end{theorem}

\begin{theorem}[Computation Complexity Reduction]\label{CCR}
The computation complexity of our method is $\mathcal{O}\left(N\times M^2\right)$, which is much smaller than the exhaustive search of $\mathcal{O}\left(\tbinom{N}{M}\times M^2\right)$.  
\end{theorem}
Theorem \ref{CCR} shows that the computation complexity of our method is independent of the number of classes. Since the dimension of neural networks is typically much larger than the class distribution vector $\bm{\alpha}_n$, the additional communication cost is almost negligible. Besides, we also prove the NP-hardness of the problem formally in Section \ref{nphardness} in the appendix.

\vspace{-5pt}
\section{Convergence Analysis}\label{Sec:Convergence}
\begin{table*}[!hbtp]
\centering
\renewcommand{\arraystretch}{1}
\setlength\tabcolsep{10pt}
\scalebox{1}{\begin{tabular}{|c|c|c|c|c|c|c|}
\hline
\multicolumn{2}{|l|}{}                                                                                  & all            & rand        & pow-d            & Fed-cucb      & Fed-CBS          \\ \hline
\multirow{3}{*}{Communication Rounds}         & $\alpha$=0.1 & 757$\pm$155   & 951$\pm$202 & 1147$\pm$130    & 861$\pm$328  & \textbf{654$\pm$96}   \\ \cline{2-7} 
                                                                                         & $\alpha$=0.2 & 746$\pm$95     & 762$\pm$105   & 741$\pm$111    & 803$\pm$220           & \textbf{475$\pm$110}   \\ \cline{2-7} 
                                                                                         & $\alpha$=0.5 & 426$\pm$67     & 537$\pm$115   & 579$\pm$140   & 1080$\pm$309           & \textbf{384$\pm$74}   \\ \hline
\multirow{3}{*}{$\mathbb{E}[QCID](10^{-2})$} & $\alpha$=0.1 & 1.01$\pm$0.01 & 8.20$\pm$0.21  & 12.36$\pm$0.26 & 7.09$\pm$2.27  & \textbf{0.62$\pm$0.20} \\ \cline{2-7} 
                                                                                         & $\alpha$=0.2 & 0.93$\pm$0.03  & 7.54$\pm$0.27 & 10.6$\pm$0.48  & 5.93$\pm$1.01 & \textbf{0.51$\pm$0.12} \\ \cline{2-7} 
                                                                                         & $\alpha$=0.5 & 0.72$\pm$0.03  & 5.87$\pm$0.24 & 7.36$\pm$0.57  & 6.47$\pm$0.77 & \textbf{0.36$\pm$0.04} \\ \hline
\end{tabular}}
\caption{The communication rounds required for targeted test accuracy and the averaged QCID values. The targeted test accuracy is $45\%$ for $\alpha=0.1$, $47\%$ for $\alpha=0.2$ and $50\%$ for $\alpha=0.5$. The results are the mean and the standard deviation over 4 different random seeds.}
\label{roundqcid1}
\vspace{-5pt}
\end{table*}

\begin{figure*}[!hbtp]
  \centering
  \includegraphics[width=1\linewidth]{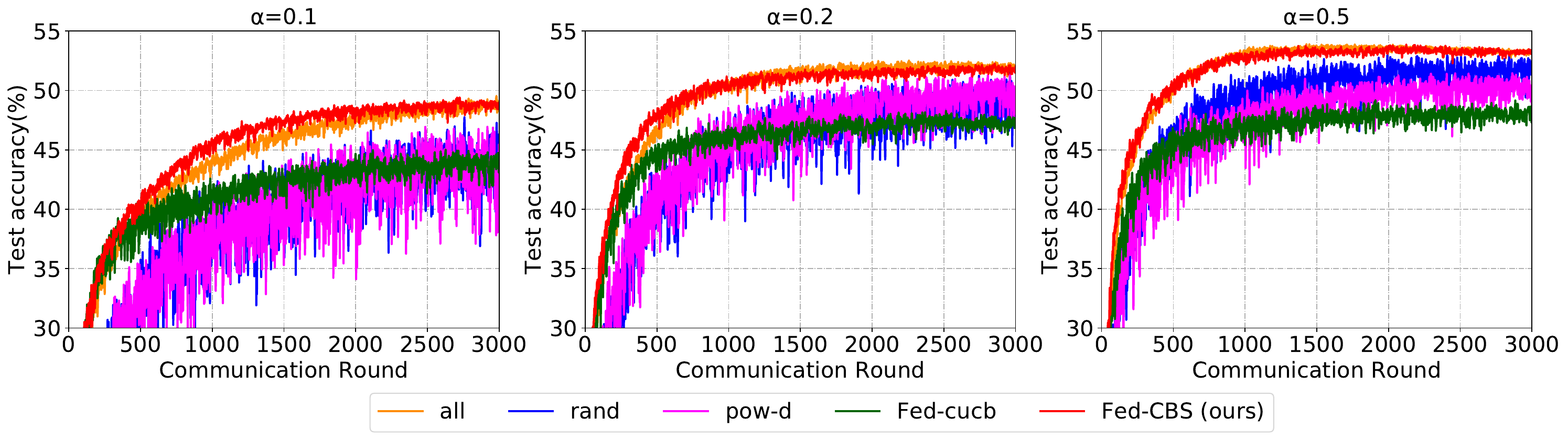}\\
  \vspace{-0.3cm}
\caption{Test accuracy on Cifar-10 under three heterogeneous settings.}
\label{fig:figcifar1}
\vspace{-15pt}
\end{figure*}
To analyze the convergence of our method, we first define our objective functions and adopt some  general assumptions. Our global objective function $\widetilde{F}>0$ can be decomposed as $\widetilde{F}=\frac{1}{B}\sum_{b=1}^B\widetilde{F}_b$, where $\widetilde{F}_b$ is the averaged loss function with respect to all the data of the $b$-th class in the global dataset. Similarly, the $n$-th client's local objective function $F_n$ can be decomposed as $F_n=\sum_{b=1}^B\alpha_{(n,b)}{F}_{n,b}$, where ${F}_{n,b}$ is the averaged loss function with respect to all the data of the $b$-th class in the $n$-th client's local dataset, and $\alpha_{(n,b)}$ is defined in Equation \ref{localdistribution}. Moreover, let $\boldsymbol{w}^{(k)}$ denote the global model parameters at the $k$-th communication round and $\boldsymbol{w}^{(0)}$ denote the initial global model parameters. If not stated explicitly, $\nabla$ denotes $\nabla_w$ throughout the paper.

\begin{assumption}[Smoothness]\label{ass1}
The global objective function $\widetilde{F}$  and each client's averaged loss function ${F}_{n,b}$ are Lipschitz smooth, \textit{i.e.} {\small $\left\|\nabla \widetilde{F}(\boldsymbol{w})-\nabla\widetilde{F}(\boldsymbol{w}^\prime)\right\| \leq L_{\widetilde{F}}\|\boldsymbol{w}-\boldsymbol{w}^\prime\|$} and   ${\small \left\|\nabla {F}_{n,b}(\boldsymbol{w})-\nabla{F}_{n,b}(\boldsymbol{w}^\prime)\right\| \leq L_{n,b}\|\boldsymbol{w}-\boldsymbol{w}^\prime\|, \forall n,b,\boldsymbol{w},\boldsymbol{w}^\prime}$. 
\end{assumption}

\begin{assumption}[Unbiased Gradient and Bounded Variance]\label{ass2}
The stochastic gradient $g_{n}$ at each client is an unbiased estimator of the local gradient: ${\small\mathbb{E}_{\xi}\left[g_{n}(\boldsymbol{w} \mid \xi)\right]=\nabla F_{n}(\boldsymbol{w})}$, with bounded variance ${\small\mathbb{E}_{\xi}\left[\left\|g_{n}(\boldsymbol{w} \mid \xi)-\nabla F_{n}(\boldsymbol{w})\right\|^{2}\right] \leq \sigma^{2},\forall  \boldsymbol{w}}$, where ${\small \sigma^{2} \geq 0}$.
\end{assumption}

\begin{assumption}[Bounded Dissimilarity]\label{ass3}
There exist two non-negative constants ${ \delta \geq 1, \gamma^{2} \geq 0}$ such that ${\small \sum_{b=1}^{B} \frac{1}{B}\left\|\nabla \widetilde{F}_b(\boldsymbol{w})\right\|^{2} \leq \delta\left\|\sum_{b=1}^{B} \frac{1}{B} \nabla \widetilde{F}_{b}(\boldsymbol{w})\right\|^{2}+\gamma^{2}}, {\small \forall \boldsymbol{w}.}$
\end{assumption}

\begin{assumption}[Class-wise Similarity]\label{ass4}
For each class $b$, the discrepancy between the gradient of global averaged loss function and the local one is bounded by some constant in $l^2$ norm. That means, for every $n$ and $b$, we have ${\small \left\|\nabla\widetilde{F}_b(\boldsymbol{w})-\nabla F_{n,b}(\boldsymbol{w})\right\|^{2} \leq \kappa^2_{n,b}}$, ${\small \forall \boldsymbol{w}}$. 
\end{assumption}
\vspace{-5pt}
Assumptions \ref{ass1}, \ref{ass2} and \ref{ass3} have been widely adopted in previous literature on the theoretical analysis of FL \cite{li2019convergence,Cho2020ClientSI,wang2020tackling}.  Assumption \ref{ass4} is based on the similarity among the data from the same class. Similar to the standard setting \cite{wang2020tackling}, the convergence of our algorithm is measured by the norm of the gradients, stated in Theorem~\ref{theo:conv}.

\begin{theorem}
 Under Assumptions \ref{ass1} to \ref{ass4}, if the total communication rounds $K$ is pre-determined and the learning rate is set as $\eta=\frac{s}{10L\sqrt{\tau(\tau-1)K}}$, where $s<1$, $L=\max_{\{n,b\}}L_{n,b}$ and $\tau$ is the number of local update iterations, the minimal gradient norm of $\widetilde{F}$ is bounded as:
\vspace{-5pt}
{ \begin{align*}
    &\min _{k \leq K}\left\|\nabla \widetilde{F}\left(\boldsymbol{w}^{(k)}\right)\right\|^{2} \leq \frac{1}{V}[\frac{\sigma^2s^2}{25\tau K}+\frac{sL_{\widetilde{F}}\sigma^2}{10L\sqrt{\tau(\tau-1)K}}\nonumber\\
    &+5\kappa^2+\frac{10L\sqrt{\tau(\tau-1)}\widetilde{F}\left(\boldsymbol{w}^{(0)}\right)}{s \sqrt{K}}+\gamma^2\mathbb{E}[QCID]],
\end{align*}\label{theo:conv}
\vspace{-5pt}
where $V=\frac{1}{3}-\delta B \mathbb{E}[QCID]$ and  $\kappa=\max_{\{n,b\}}\kappa_{n,b}$.}
\end{theorem}
If the class-imbalance in client selection is reduced,  $\mathbb{E}[QCID]$ will decrease. Consequently,  $\frac{1}{V}$ and $\frac{\mathbb{E}[QCID]}{V}$ will also decrease, making the convergence bound on the right side tighter\footnote{Theorem \ref{theo:conv} requires the $\beta_M$ in our method to be large enough to make $\mathbb{E}[QCID]<\frac{1}{3\delta B}$ according to Theorem \ref{CIR}. How to explicitly derive a lower bound for $\beta_M$ is also very interesting and we leave it as a theoretical future work.}. Therefore, Theorem \ref{theo:conv} not only provides a convergence guarantee for Fed-CBS, but also proves the class-imbalance reduction in client selection could benefit FL, {\it i.e.}, more class-balance leads to faster convergence. 
\vspace{-10pt}
\section{Experiments}\label{Sec:experiments}
We conduct thorough experiments on three public benchmark datasets, CIFAR-10 \cite{cifar10}, Fashion-MNIST \cite{xiao2017fashionmnist} and FEMNIST in the Leaf Benchmark \cite{caldas2018leaf}.  In all the experiments, we simulate cross-device federated learning (CDFL), where the system runs with a large number of clients with only a fraction of them available in each communication round, and we make client selections on those available clients. The results show that our method can achieve faster and more stable convergence compared with four baselines: random selection (rand), Power-of-choice Selection Strategy (pow-d)~\citep{Cho2020ClientSI}, the method in  \citet{yang2020federated} (Fed-cucb), and the ideal setting where we select all the available clients (all). To compare them efficiently in the main text, we present the results from Cifar-10 where the whole dataset is divided to 200 (or 120) clients, since we need to engage all the clients for the ideal setting. To simulate more realistic settings where there are thousands of clients, we conduct our method on FEMNIST in the Leaf Benchmark with more then 3000 clients. Due to the space limit, we move the results of FEMNIST, Fashion-MNIST, and the ablation studies to Section \ref{Sec:resultfmnist} $\&$ \ref{Sec:ablation} in the Appendix. For Fashion-MNIST, we adopt FedNova \citep{wang2020tackling} to show that our method can be organically integrated with existing orthogonal works which aim at improving FL.
\vspace{-5mm}
\paragraph{Experiment Setup} We target cross-device settings where the devices are resource-constrained, i.e., most of the devices do not have sufficient computational power and memory to support the training of large models. Therefore,  we adopt a compact model with two convolutional layers followed by three fully-connected layers and FedAvg \citep{mcmahan2016communication} as the FL optimizer. The batch size is $50$ for each client. In each communication round, all of them conduct the same number of local updates, which allows the client with the largest local dataset to conduct 5 local training epochs. In our method, we set the $\beta_m=m$, $\gamma=10$ and $L_{b}=10^{-20}$. The local optimizer is SGD with a weight decay of 0.0005. The learning rate is 0.01 initially and the decay factor is 0.9992. We terminate the FL training after 3000 communication rounds and then evaluate the model's performance on the test dataset of CIFAR-10. More details of the experiment setup are listed in Section \ref{Sec:settingsec5}.

\vspace{-5pt}
\subsection{Results for Class-Balanced Global Datasets}\label{sec:RCBGD}

\begin{table*}[ht]
\centering
\renewcommand{\arraystretch}{1}
\setlength\tabcolsep{14.5pt}
\scalebox{1}{
\begin{tabular}{|c|c|c|c|c|c|c|}
\hline
\multicolumn{2}{|l|}{}                                                                          & all            & rand         & pow-d            & Fed-cucb                     & Fed-CBS                       \\ \hline
\multirow{2}{*}{Case 1}                                                            & 3:1 & 55.17$\pm$0.94 & 50.99$\pm$0.97 & 53.51$\pm$0.34 & 55.11$\pm$0.26                      &\textbf{56.86$\pm$0.34}          \\ \cline{2-7} 
                                                                                          & 5:1 & 50.93$\pm$1.64 & 47.36$\pm$2.34 & 52.73$\pm$1.85 &    53.75$\pm$0.58              & \textbf{54.94$\pm$0.73} \\ \hline
\multirow{2}{*}{Case 2}                                                            & 3:1 & 54.01$\pm$0.60 & 50.81$\pm$2.03 & 53.98$\pm$1.87 & 54.48$\pm$1.31 & \textbf{57.71$\pm$0.50}  \\ \cline{2-7} 
                                                                                          & 5:1 & 50.42$\pm$1.27 & 48.33$\pm$3.03 & 53.54$\pm$1.18 & 53.38$\pm$1.48 & \textbf{57.99$\pm$0.46}  \\ \hline
\end{tabular}}
\caption{Best test accuracy for our method and other four baselines. }
\label{imbalancebest}
\vspace{-10pt}
\end{table*}

In this experiment, we set 200 clients in total with a class-balanced global dataset. The non-IID data partition among clients is based on a Dirichlet distribution parameterized by the concentration parameter $\alpha$ in \citet{Hsu2019MeasuringTE}. Roughly speaking, as $\alpha$ decreases, the data distribution will become more non-iid. In each communication round, we uniformly and randomly set 30$\%$ of them (i.e., 60 clients) available and select 10 clients from those 60 available ones to participate in the training. 

As shown in Table \ref{roundqcid1}, our method can achieve the lowest $QCID$ value compared with other client selection strategies. As a benefit of successfully reducing the class-imbalance, our method outperforms the other three baseline methods and achieves comparable performance to the ideal setting where all the available clients are engaged in training. As shown in Table \ref{roundqcid1} and Figure \ref{fig:figcifar1}, our method can achieve faster and more stable convergence. The enhancement in stability can also be perceived as a reduction in gradient variance, a concept that has been explored in previous studies \cite{johnson2013accelerating,zhang2020variance,defazio2014saga,zhao2018selfadversarially,chatterji2018theory}. It is also worth noting that due to the inaccurate distribution estimation and the limitations of the greedy method discussed in Section \ref{sec:relatedwork}, the performance of Fed-cucb is much worse than ours.
\vspace{-5pt}
\subsection{Results for Class-Imbalanced Global Datasets}\label{RCIGD}
In real-world settings, the global dataset of all the clients is not always class-balanced. Hence, we investigate two different cases to show the superiority of our method and provide more details of their settings in Section \ref{Sec:settingCase12}. To simplify the construction of a class-imbalanced global dataset, each client only has one class of data with the same quantity. We report the best test accuracy in Table \ref{imbalancebest} and present the corresponding $QCID$ values in Section \ref{Sec:QCIDcase12}.

\subsubsection{Case 1: Uniform Availability}
\textbf{Settings.} There are 120 clients in total, and the global dataset of these 120 clients is class-imbalanced. To measure the degree of class imbalance, we let the global dataset have the same amount of $n_1$ data samples for five classes and the same amount of $n_2$ data samples for the other five classes. The ratio $r$ between $n_1$ and $n_2$ is respectively set to $3:1$ and $5:1$. In each communication round, we uniformly set 30$\%$ of them (\textit{i.e.}, 36 clients) available with replacement and select 10 clients to participate in the training.

As shown in Table \ref{imbalancebest} and Figure \ref{fig:case1}, our method can achieve faster and more stable convergence, and it even achieves slightly better performance than the ideal setting where all the available clients are engaged. The performance of Fed-cucb \citep{yang2020federated} is better than the results on the class-balanced global dataset, which is partly due to the simplicity of each client's local dataset composition in our experiments. The third line in Figure \ref{RealEstimated} indicates Fed-cucb can accurately estimate this simple type of label distribution.

\begin{figure}[ht]
\centering
  \includegraphics[width=1\linewidth]{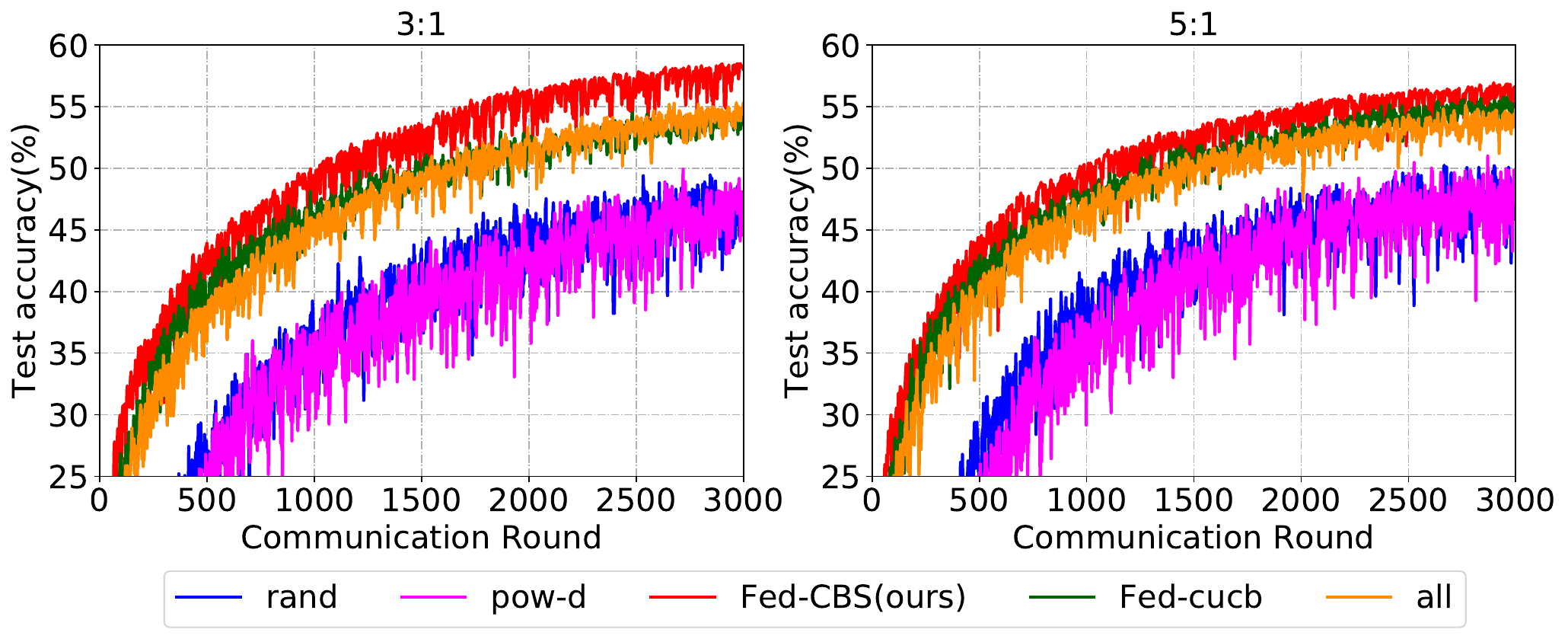}\\
  \vspace{-10pt}
\caption{Test accuracy on Cifar-10 with class-imbalanced global dataset in Case 1. }
\vspace{-10pt}
\label{fig:case1}
\end{figure}

\subsubsection{Case 2: Non-uniform Availability}
\textbf{Settings.} There are 200 clients in total. In each communication round, 30$\%$ of them (\textit{i.e.}, 60 clients) are set available uniformly in each training round with replacement. By non-uniformly setting the availability, the global dataset of those 60 available clients is always class-imbalanced. To measure the degree of class imbalance, we make the global dataset have the same amount of $n_1$ data samples for the five classes and have the same amount of $n_2$ data samples for the other five classes. The ratio $r$ between $n_1$ and $n_2$ is set to $3:1$ and $5:1$. We select 10 clients to participate in the training.

As shown in Table \ref{imbalancebest} and Figure \ref{fig:case1}, our method consistently achieves higher test accuracy and more stable convergence, and it also outperforms the ideal setting where all the available clients are engaged. Since the global dataset of the available 60 clients in each communication round is always class-imbalanced, engaging all of them is not the optimal selection strategy in terms of test accuracy.
\vspace{-10pt}
\begin{figure}[ht]
  \centering
  \includegraphics[width=1\linewidth]{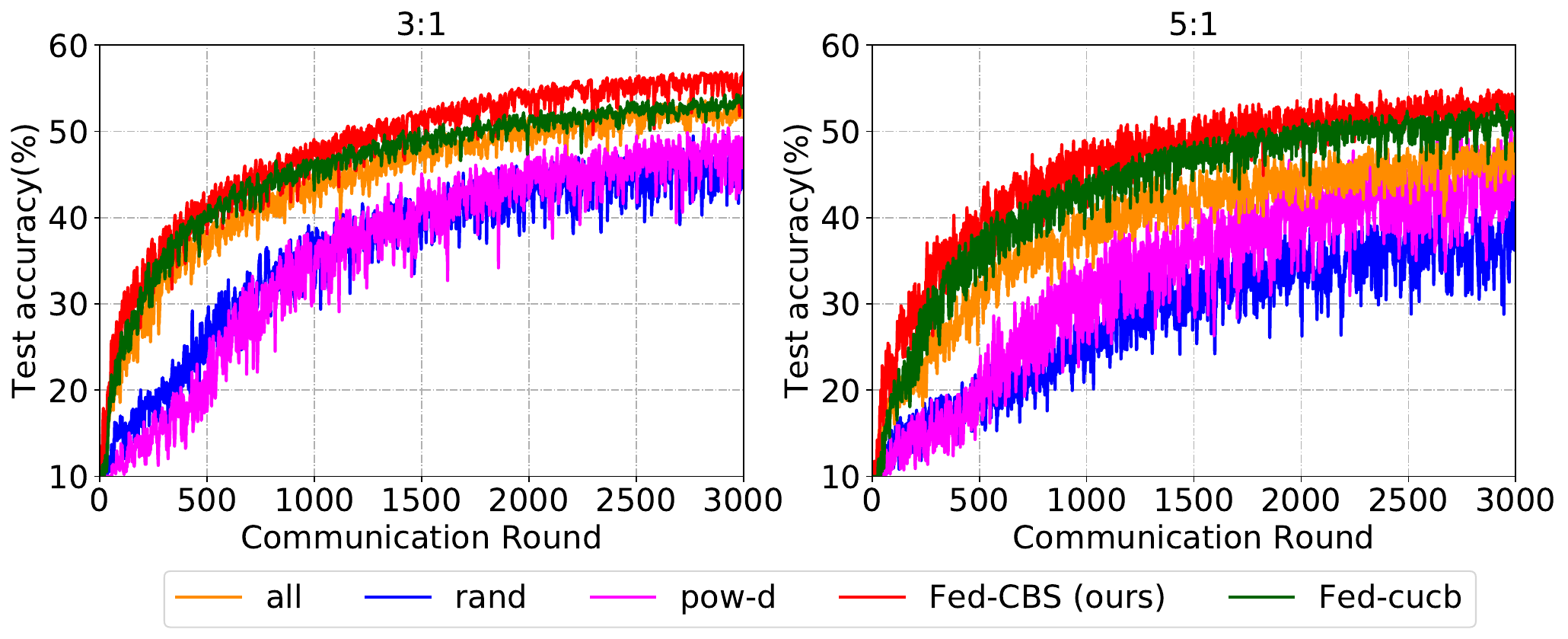}\\
  \vspace{-0.3cm}
\caption{Test accuracy on Cifar-10 with class-imbalanced global dataset in Case 2.}
\label{fig:case2}
\end{figure}
\vspace{-10pt}
\section{Conclusion}\label{sec:conclusion}
We unveil the essential reason for performance degradation on non-IID data with random client selection strategy in FL training, \textit{i.e.}, the class-imbalance. Motivated by this insight, we propose an efficient heterogeneity-aware client sampling mechanism, Fed-CBS. Extensive experiments validate that Fed-CBS significantly outperforms the status quo approaches and yields comparable or even better performance than the ideal setting where all the available clients participate in the training. We also provide the theoretical convergence guarantee of Fed-CBS. Our mechanism has numerous potential applications, including medical classification tasks. In addition, since Fed-CBS is orthogonal to most existing work to improve FL on non-IID data, it can be integrated with them to further improve the performance.

\paragraph{Acknowledgments}
This work is supported in part by the grants CNS-2112562, IIS-2140247, CNS-1822085 and IIS-2223292. We thank Eric Yeats, Shihan Lin and Taoan Huang for the valuable discussion and thank all reviewers for their valuable comments.

\nocite{langley00}

\bibliography{example_paper}
\bibliographystyle{icml2023}

\newpage
\appendix
\onecolumn

\section{Privacy Protection in the framework} \label{selectorprivacy}

\subsection{Proof of Theorem \ref{theoremprivacy}}

\begin{proof}

By the definitions of $\{\bm{\alpha_i}\}$, we define the following matrix $\mA_{\alpha}$

\begin{align*}
    \mA_{\alpha}\triangleq \begin{bmatrix}q_1\bm{\alpha}_1 \\...\\ q_n\bm{\alpha}_n \\ ... \\ q_N\bm{\alpha}_N \end{bmatrix} = \begin{bmatrix}q_1\alpha_{(1,1)}& q_1\alpha_{(1,2)}&...&q_1\alpha_{(1,b)}&... &q_1\alpha_{(1,B)} \\  ...\\  q_n\alpha_{(n,1)}& q_n\alpha_{(n,2)}&...&q_n\alpha_{(n,b)}&... &q_n\alpha_{(n,B)}  \\...\\  q_N\alpha_{(N,1)}& q_N\alpha_{(N,2)}&...&q_N\alpha_{(N,b)}&... &q_N\alpha_{(N,B)}\end{bmatrix} 
\end{align*}

By the definitions of $\mS$, we have 
\begin{equation}\label{StoA}
    \mS= \mA_{\alpha}\cdot \mA^\intercal_{\alpha}
\end{equation}
To derive the exact values of $\{\bm{\alpha_i}\}$ based on $\mS$, we need to solve the problem \ref{StoA}. However, given $\mS$, the $\mA_{\alpha}$ which satisfies $\mS= \mA_{\alpha}\cdot \mA^\intercal_{\alpha}$ is not unique. If $\bar{\mA}_{\alpha}$ is a solution to the problem \ref{StoA}, then for any orthogonal matrix $\mQ$ \textit{i.e.} $\mQ\cdot \mQ^\intercal = \mI$ where the $\mI$ is the identity matrix, the new matrix $\bar{\mA}_{\alpha} \cdot \mQ$ is also solution to the problem \ref{StoA}. This is because 

\begin{align*}
    \bar{\mA}_{\alpha}\cdot \mQ \cdot (\bar{\mA}_{\alpha} \cdot \mQ)^\intercal=\bar{\mA}_{\alpha}\cdot \mQ \cdot \mQ^\intercal \cdot \bar{\mA}^\intercal_{\alpha}=\bar{\mA}_{\alpha}\cdot  \bar{\mA}^\intercal_{\alpha}=\mS
\end{align*}

Hence, the $\mA_{\alpha}$ which satisfies $\mS= \mA_{\alpha}\cdot \mA^\intercal_{\alpha}$ is not unique and we finish our proof.
\end{proof}
To understand the Theorem \ref{theoremprivacy}, we provide the following example. We can conduct the following permutation on the columns of $\mA_\alpha$ (\textit{i.e.} moving the first column to the place before the last column), we can derive a new matrix $\bar{\mA}_\alpha$. 
\begin{align*}
    \bar{\mA}_{\alpha}\triangleq \begin{bmatrix} q_1\alpha_{(1,2)}&...&q_1\alpha_{(1,b)}&... &q_1\alpha_{(1,1)}&q_1\alpha_{(1,B)} \\  ...\\   q_n\alpha_{(n,2)}&...&q_n\alpha_{(n,b)}&... &q_n\alpha_{(n,1)}&q_n\alpha_{(n,B)}  \\...\\   q_N\alpha_{(N,2)}&...&q_N\alpha_{(N,b)}&... &q_N\alpha_{(N,1)}&q_N\alpha_{(N,B)}\end{bmatrix} 
\end{align*}

We can find that $\bar{\mA}_\alpha$ also satisfies $\mS= \bar{\mA}_{\alpha}\cdot \bar{\mA}^\intercal_{\alpha}$. Actually, there are also many other permutations that can derive the solutions to the problem \ref{StoA}. Hence, in our framework shown in \ref{FHE}, the selector can not estimate the exact label distribution of the clients.

\subsection{An Example of Deriving $\mS$ Using FHE}\label{FHEexample}

\begin{figure}[ht]
\begin{center}
\includegraphics[width=1.0\linewidth]{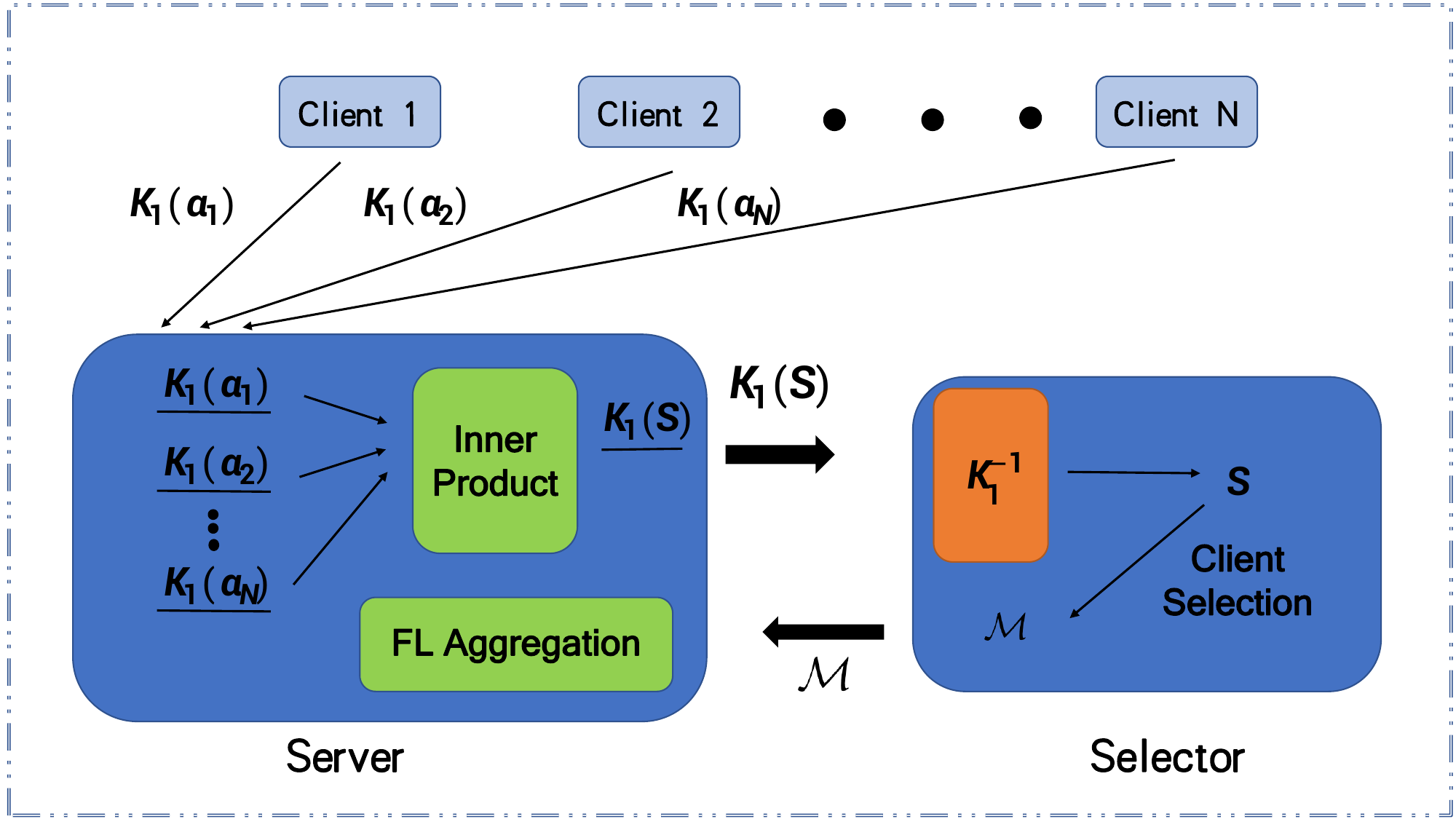}
\end{center}
\vspace{-10pt}
\caption{An example of FHE to securely transmit ${S}$. }
\vspace{-10pt}\label{FHE}
\end{figure}
FHE~\cite{bgvfhe,fvfhe,halevi2015bootstrapping} enables an untrusted party to perform computation (addition and multiplication) on encrypted data. In Figure \ref{FHE}, we provide a framework as an example to show it is possible to derive ${S}$ without knowing the values of local label distributions $\{\bm{\alpha_i}\}$ using FHE. Our framework can be realized using off-the-shelf FHE libraries such as~\cite{seal}. 

There is a selector in our example. It is usually from a third party and keeps a unique private key, denoted by $K_1^{-1}$. The corresponding public key is denoted by $K_1$. In the confidential transmission between server and clients, each client first uses $K_1$ to encrypt their label distribution vector $\bm{\alpha_k}$ as $K_1(\bm{\alpha_k})$, and transmits it to the server. Since only the server has access to $K_1(\bm{\alpha_k})$, no one else including the selector can decrypt it and get $\bm{\alpha_k}$. When the server gets all $K_1(\bm{\alpha_k})$, it will conduct FHE computation to get the matrix $K_1(S)=K_1(\{\bm{\alpha}_i^T\bm{\alpha}_j\}_{ij})=\{K_1(\bm{\alpha}_i)^TK_1(\bm{\alpha}_j)\}_{ij}$. Then the server transmits the $K_1(S)$ to selector, and selector uses $K_1^{-1}$ to access the final result $S$. Since only the selector has $K_1^{-1}$, only it knows $S$. After that, the selector will conduct client selection following some strategy to derive the result $\mathcal{M}$ and transmit it back to the server. At last, the server will collect the model parameters of the clients in $\mathcal{M}$ and conduct FL aggregation. In the whole process, the server, selector or any other clients except client $n$ can not get $\alpha_n$. Furthermore, all clients and servers have no access to the inner product results $S$, which prevents malicious clients or servers from inferring the label distributions of the other clients.

The server, selector or any other clients except client $n$ can not get $\alpha_n$, which protects the privacy of the clients. Furthermore, only the clients have no access to the inner product results $S$, which prevents malicious clients or servers from inferring the label distributions of the other clients. We also prove that it is impossible even for the selector to derive $\{\bm{\alpha_i}\}$ from $S$ with theorem \ref{theoremprivacy}.

\section{Proof of Theorem \ref{theorem1}, \ref{CIR}, \ref{CCR} and \ref{theo:conv}}\label{Sec:proof}
\subsection{Proof of Theorem \ref{theorem1}}
\begin{proof}
\begin{align*}
    \textit{QCID}(\mathcal{M}) &= \sum_{b=1}^B(\frac{\sum_{n\in \mathcal{M} }q_n \alpha_{(n,b)}}{\sum_{n\in \mathcal{M} } q_n}-\frac{1}{B})^2\\
    &=\sum_{b=1}^B(\frac{\sum_{n\in \mathcal{M} }q_n \alpha_{(n,b)}}{\sum_{n\in \mathcal{M} }q_n})^2-2*\frac{1}{\sum_{n\in \mathcal{M} }q_n}*\frac{1}{B}*\sum_{n\in \mathcal{M} }q_n+B*\frac{1}{B^2}\\
    &=\sum_{b=1}^B(\frac{\sum_{n\in \mathcal{M} } q_n\alpha_{(n,b)}}{\sum_{n\in \mathcal{M} }q_n})^2-\frac{1}{B}\\
    &=\frac{1}{(\sum_{n\in \mathcal{M} }q_n)^2}\sum_{n\in \mathcal{M},n^\prime \in \mathcal{M}} q_nq_n^\prime(\sum_{b=1}^B \alpha_{(n,b)} \alpha_{(n^\prime,b)})-\frac{1}{B}\\
    &=\frac{1}{(\sum_{n\in \mathcal{M} }q_n)^2} \sum_{n\in \mathcal{M},n^\prime \in \mathcal{M}} q_nq_n^\prime\bm{\alpha}_n\bm{\alpha}_{n^\prime}^T-\frac{1}{B}
\end{align*}
\end{proof}

\subsection{Proof of Theorem \ref{CIR}}

\begin{proof}

To select $M$ clients from $N$ available clients, there are $\tbinom{N}{M}$ different choices to construct $\mathcal{M}$, denoted by $\mathcal{M}^{(1)}, \mathcal{M}^{(2)}, ...,\mathcal{M}^{(\tbinom{N}{M})}$, respectively. Let $x_i \triangleq \text{\textit{QCID}}(\mathcal{M}^{(i)})$ and $\bar{N}  \triangleq (\tbinom{N}{M})$.
Then we have 
\begin{align*}
    \mathbb{E}_{\mathcal{M}\sim P_{\beta_M}} QCID(\mathcal{M})=x_1\frac{\frac{1}{x_1^{\beta_{M}}}}{\frac{1}{x_1^{\beta_{M}}}+\frac{1}{x_2^{\beta_{M}}}+...+\frac{1}{x_{\bar{N}}^{\beta_{M}}}}+x_2\frac{\frac{1}{x_2^{\beta_{M}}}}{\frac{1}{x_1^{\beta_{M}}}+\frac{1}{x_2^{\beta_{M}}}+...+\frac{1}{x_{\bar{N}}^{\beta_{M}}}}+...+x_{\bar{N}}\frac{\frac{1}{x_{\bar{N}}^{\beta_{M}}}}{\frac{1}{x_1^{\beta_{M}}}+\frac{1}{x_2^{\beta_{M}}}+...+\frac{1}{x_{\bar{N}}^{\beta_{M}}}}
\end{align*}
\begin{align*}
    \text{And }\mathbb{E}_{\mathcal{M}\sim P_{rand}} QCID(\mathcal{M})=\frac{1}{\bar{N}}(x_1+x_2+...+x_{\bar{N}})
\end{align*}

Without loss of generality, we assume $x_1\leq x_2 \leq ...\leq x_{\bar{N}}$ and define the following $y_i$ for the notation simplicity:

\begin{equation}
y_i=
\begin{cases}
\frac{1}{x_i^{\beta_{M}}}& \text{if}\quad 0\leq i \leq \bar{N}\\
\frac{1}{x_{i-\bar{N}}^{\beta_{M}}}& \text{if}\quad \bar{N} < i \leq 2\bar{N}-1
\end{cases}
\end{equation}

Now we calculate the following ratio:

\begin{align*}
    \frac{\mathbb{E}_{\mathcal{M}\sim P_{\beta_M}} QCID(\mathcal{M})}{\mathbb{E}_{\mathcal{M}\sim P_{rand}} QCID(\mathcal{M})}&=\frac{\bar{N}(x_1\frac{1}{x_1^{\beta_{M}}}+x_2\frac{1}{x_2^{\beta_{M}}}+...+x_{\bar{N}}\frac{1}{x_{\bar{N}}^{\beta_{M}}})}{(x_1+x_2+...+x_{\bar{N}})(\frac{1}{x_1^{\beta_{M}}}+\frac{1}{x_2^{\beta_{M}}}+...+\frac{1}{x_{\bar{N}}^{\beta_{M}}})}\\
    &=\frac{\sum_{j=1}^{\bar{N}}(\sum_{i=1}^{\bar{N}}x_i\frac{1}{x_i^{\beta_{M}}})}{\sum_{j=1}^{\bar{N}}(\sum_{i=1}^{\bar{N}}x_iy_{j+i-1})}
\end{align*}

Since we assume that $x_1\leq x_2 \leq ...\leq x_{\bar{N}}$, we have $\frac{1}{x_1^{\beta_{M}}}\geq \frac{1}{x_2^{\beta_{M}}} \geq ...\geq \frac{1}{x_{\bar{N}}^{\beta_{M}}}$. Besides, it is easy to find $x_i$ and $x_{i^\prime}$ satisfying $x_i \neq x_{i^\prime}$. Then for each $1\leq j \leq \bar{N}$, according to the rearrangement inequality, we have 
\begin{align*}
    &\sum_{i=1}^{\bar{N}}x_i\frac{1}{x_i^{\beta_{M}}}< \sum_{i=1}^{\bar{N}}x_iy_{j+i-1}\Rightarrow   \sum_{j=1}^{\bar{N}}\sum_{i=1}^{\bar{N}}x_i\frac{1}{x_i^{\beta_{M}}}< \sum_{j=1}^{\bar{N}}\sum_{i=1}^{\bar{N}}x_iy_{j+i-1} \Rightarrow \\
    &\frac{\mathbb{E}_{\mathcal{M}\sim P_{\beta_M}} QCID(\mathcal{M})}{\mathbb{E}_{\mathcal{M}\sim P_{rand}} QCID(\mathcal{M})}<1 \Rightarrow {\mathbb{E}_{\mathcal{M}\sim P_{\beta_M}} QCID(\mathcal{M})}<{\mathbb{E}_{\mathcal{M}\sim P_{rand}} QCID(\mathcal{M})}
\end{align*}

Similarly, for $\beta_M^\prime$ such that $\beta_M^\prime \geq \beta_M$, denote $\beta_M^\prime=\beta_M+\Delta \beta$. We have 

{\footnotesize 
\begin{align*}
    &\mathbb{E}_{\mathcal{M}\sim P_{\beta_M^\prime}} QCID(\mathcal{M})=\\
    &x_1\frac{\frac{1}{x_1^{\beta_{M}^\prime}}}{\frac{1}{x_1^{\beta_{M}^\prime}}+\frac{1}{x_2^{\beta_{M}^\prime}}+...+\frac{1}{x_{\bar{N}}^{\beta_{M}^\prime}}}+x_2\frac{\frac{1}{x_2^{\beta_{M}^\prime}}}{\frac{1}{x_1^{\beta_{M}^\prime}}+\frac{1}{x_2^{\beta_{M}^\prime}}+...+\frac{1}{x_{\bar{N}}^{\beta_{M}^\prime}}}+...+x_{\bar{N}}\frac{\frac{1}{x_{\bar{N}}^{\beta_{M}^\prime}}}{\frac{1}{x_1^{\beta_{M}^\prime}}+\frac{1}{x_2^{\beta_{M}^\prime}}+...+\frac{1}{x_{\bar{N}}^{\beta_{M}^\prime}}}=\\
    &\frac{x_1\frac{1}{x_1^{\beta_M+\Delta \beta}}}{\frac{1}{x_1^{\beta_M+\Delta \beta}}+\frac{1}{x_2^{\beta_M+\Delta \beta}}+...+\frac{1}{x_{\bar{N}}^{\beta_M+\Delta \beta}}}+\frac{x_2\frac{1}{x_2^{\beta_M+\Delta \beta}}}{\frac{1}{x_1^{\beta_M+\Delta \beta}}+\frac{1}{x_2^{\beta_M+\Delta \beta}}+...+\frac{1}{x_{\bar{N}}^{\beta_M+\Delta \beta}}}+...+\frac{x_{\bar{N}}\frac{1}{x_{\bar{N}}^{\beta_M+\Delta \beta}}}{\frac{1}{x_1^{\beta_M+\Delta \beta}}+\frac{1}{x_2^{\beta_M+\Delta \beta}}+...+\frac{1}{x_{\bar{N}}^{\beta_M+\Delta \beta}}}
\end{align*}}

Now we calculate the following ratio: 
\begin{align*}
    \frac{\mathbb{E}_{\mathcal{M}\sim P_{\beta^{\prime}_M}} QCID(\mathcal{M})}{\mathbb{E}_{\mathcal{M}\sim P_{\beta_M}} QCID(\mathcal{M})}&=\frac{\left(\frac{x_1}{x_1^{\beta_M+\Delta \beta}}+\frac{x_2}{x_2^{\beta_M+\Delta \beta}}+...+\frac{x_N}{x_N^{\beta_M+\Delta \beta}}\right)\left(\frac{1}{x_1^{\beta_{M}}}+\frac{1}{x_2^{\beta_{M}}}+...+\frac{1}{x_{\bar{N}}^{\beta_{M}}}\right)}{\left(\frac{1}{x_1^{\beta_M+\Delta \beta}}+\frac{1}{x_2^{\beta_M+\Delta \beta}}+...+\frac{1}{x_{\bar{N}}^{\beta_M+\Delta \beta}} \right)\left(\frac{x_1}{x_1^{\beta_{M}}}+\frac{x_2}{x_2^{\beta_{M}}}+...+\frac{x_{\bar{N}}}{x_{\bar{N}}^{\beta_{M}}}\right)}\\
    &=\frac{\sum_{1\leq i \leq j\leq N}\frac{1}{x_i^{\beta_{M}}x_j^{\beta_{M}}}\left( \frac{x_i}{x_i^{\Delta \beta}}+\frac{x_j}{x_j^{\Delta \beta}}\right)}{\sum_{1\leq i \leq j\leq N}\frac{1}{x_i^{\beta_{M}}x_j^{\beta_{M}}}\left( \frac{x_i}{x_j^{\Delta \beta}}+\frac{x_j}{x_i^{\Delta \beta}}\right)}
\end{align*}

Since we assume that $x_1\leq x_2 \leq ...\leq x_{\bar{N}}$, we have $\frac{1}{x_1^{\Delta \beta}}\geq \frac{1}{x_2^{\Delta \beta}} \geq ...\geq \frac{1}{x_{\bar{N}}^{\Delta \beta}}$.  Then for each $1\leq i \leq j\leq N$, according to the rearrangement inequality, we have
\begin{align*}
    \frac{x_i}{x_i^{\Delta \beta}}+\frac{x_j}{x_j^{\Delta \beta}} \leq \frac{x_i}{x_j^{\Delta \beta}}+\frac{x_j}{x_i^{\Delta \beta}}
\end{align*}

Furthermore, among all the $(x_i, x_j)$ pairs, it is easy to find one $(x_i, x_{i^\prime})$ such that it satisfies $x_i \neq x_{i^\prime}$. Thus we have 
\begin{align*}
    \frac{x_i}{x_i^{\Delta \beta}}+\frac{x_{i^\prime}}{x_{i^\prime}^{\Delta \beta}} < \frac{x_i}{x_{i^\prime}^{\Delta \beta}}+\frac{x_{i^\prime}}{x_i^{\Delta \beta}}
\end{align*}

Consequently, we have
\begin{align*}
    &\sum_{1\leq i \leq j\leq N}\frac{1}{x_i^{\beta_{M}}x_j^{\beta_{M}}}\left( \frac{x_i}{x_i^{\Delta \beta}}+\frac{x_j}{x_j^{\Delta \beta}}\right)<\sum_{1\leq i \leq j\leq N}\frac{1}{x_i^{\beta_{M}}x_j^{\beta_{M}}}\left( \frac{x_i}{x_j^{\Delta \beta}}+\frac{x_j}{x_i^{\Delta \beta}}\right)\Rightarrow\\
    &\frac{\mathbb{E}_{\mathcal{M}\sim P_{\beta^{\prime}_M}} QCID(\mathcal{M})}{\mathbb{E}_{\mathcal{M}\sim P_{\beta_M}} QCID(\mathcal{M})}<1 \Rightarrow \mathbb{E}_{\mathcal{M}\sim P_{\beta^{\prime}_M}} QCID(\mathcal{M})<\mathbb{E}_{\mathcal{M}\sim P_{\beta_M}} QCID(\mathcal{M})
    \end{align*}

\end{proof}

\subsection{Proof of NP-hardness}\label{nphardness}

We provide the following proof to prove the NP-hardness. First, we need to clarify the definitions of the following three problems.

\textbf{Problem 1}: We need to select M clients among N clients such that the grouped dataset of these M clients is class-balanced. There are B (B 2) classes in total.
Our goal is to prove the NP-hardness of Problem 1.

\textbf{Problem 2}: We need to select $\mathrm{N}$ clients among $2 \mathrm{~N}$ clients such that the group dataset of these N clients is class-balanced. There are 2 classes in total. We denote the distribution of the local dataset of the $n$-th client as $\left[x_n, y_n\right]$, where $x_n$ and $y_n$ are non-negative integers.

Problem 2 is a particular case of Problem 1. If we can prove the NP-hardness of Problem 2, then Problem 1 is also NP-hard.

\textbf{Problem 3 (Partition problem)}: Deciding whether a given multiset $S$ of $K$ positive integers can be partitioned into two subsets $S_1$ and $S_2$ such that the sum of the numbers in $S_1$ equals the sum of the numbers in $S_1$. We denote the $S$ as $\left\{s_1, s_2, \ldots, s_K\right\}$

It is well-known that the Partition problem is an NP-complete problem. Hence the overall idea of our proof is to reduce Problem 2 to Problem 3. Then we can show that Problem 2 is NP-hard.

\begin{proof}

Case 1: We first consider the case where $K$ is an even number, where $K=2 N$.
We denote the sum of all the elements in $S$ as $W$, where $W=s_1+s_2+\ldots+s_K$. We define a new positive value $P$ as 
$$
P=\min\{ \left.\left|2 W-K s_1\right|,\left|2 W-K s_2\right|, \ldots,\left|2 W-K s_K\right|\right\}+1
$$
Now, we define the following non-negative $x_n$ and $y_n$, where $1 \leq n \leq K$
$$
\begin{array}{ll}
x_1=K s_1+P, & y_1=2 W-K s_1+P \\
x_2=K s_2+P, & y_2=2 W-K s_2+P
\end{array}
$$
$$
x_K=K s_K+P, \quad y_K=2 W-K s_K+P
$$
Now we can consider the $\left[x_1, y_1\right],\left[x_2, y_2\right], \ldots,\left[x_K, y_K\right]$ as the class distributions defined in Problem 2 . If Problem 2 is not NP-hard, we can find $N=\frac{K}{2}$ clients among the above $K$ clients such that the grouped dataset is class-balanced within polynomial time complexity. We denote those $N$ clients' distribution as $\left[\bar{x}_1, \bar{y}_1\right],\left[\bar{x}_2, \bar{y}_2\right], \ldots,\left[\bar{x}_K, \bar{y}_N\right]$. Then we denote the corresponding elements in $S$ as $\bar{s}_1, \bar{s}_2, \ldots, \bar{s}_N$.
Since it is class-balanced solution, we have
$$
\bar{x}_1+\bar{x}_2+\ldots+\bar{x}_N=\bar{y}_1+\bar{y}_2+\ldots+\bar{y}_N
$$
By summarizing all the $\bar{x}_n$ and $\bar{y}_n$. we can derive that $\left(\bar{x}_1+\bar{y}_1\right)+\left(\bar{x}_2+\bar{y}_2\right)+\ldots+\left(\bar{x}_N+\bar{y}_N\right)=N(2 W+2 P)$.
Then we have $\bar{x}_1+\bar{x}_2+\ldots+\bar{x}_N=N(W+P)$
According to the definition of $\bar{x}_1, \bar{x}_2, \ldots, \bar{x}_N$, we have
$$
\left(K \bar{s}_1+P\right)+\left(K \bar{s}_2+P\right)+\ldots+\left(K \bar{s}_N+P\right)=N(W+P)
$$
Since $K=2 N$
we have $\bar{s}_1+\bar{s}_2+\ldots+\bar{s}_N=\frac{W}{2}$. This means we can solve the Partition problem within polynomial time complexity when $K$ is an even number.
Case 2: If $K$ is an odd number, where $K=2 N-1$, we can just add an auxiliary element $s \prime=0$ to the original $S$ and derive a new set $S t=S \cup\{s \prime \}$. If Problem 2 is not NP-hard, we can follow the same process as in Case 1 to solve the Partition problem within polynomial time complexity when K is an odd number.

We know these solutions to Case $2 \&$ Case 1 conflict with the fact that the Partition problem is NP-hard. Hence, Problem 2 is NP-hard. Then Problem 1 is NP-hard, and we finish our proof.

\end{proof}

\subsection{Proof of Theorem \ref{CCR}}
\begin{proof}

According to \cite{schneider1989matrices}, we first define the principle submatrix, which is a submatrix where the set of remaining row indices is the same as the remaining set of column indices .

Before selecting the first client, we need to calculate the following value for all clients  $c_1\in \{1,2,3,...,N\}$,
\begin{align*}
    P(C_1=c_1) \propto \frac{1}{[QCID(\mathcal{M}_1)]^{\beta_1}}+\lambda \sqrt{\frac{3 \ln k}{2 T_{c_1}}}, \quad \beta_1>0.~
\end{align*}

To derive the $QCID(\mathcal{M}_1)$ for each  $c_1\in \{1,2,3,...,N\}$, according to Theorem $\ref{theorem1}$, we need to find the principle submatrix of S, denoted by $S_1$, in which the set of column indices is $\mathcal{M}_1$. Then we need to calculate the sum of all the elements in $S_1$. Since there are $N$ different values for $c_1$ and the dimension of $S_1$ is $1\times 1$, we need to conduct the computation for $N$ times.

After selecting $\mathcal{M}_1=\{c_1\}$, we need to select $c_2\in \{1,2,3,...,N\}/\mathcal{M}_1$ to form $\mathcal{M}_2=\mathcal{M}_1 \bigcup \{c_2\}$.

Before selecting the second client, we need to calculate the following value for all the  $\mathcal{M}_2=\{c_1,c_2\}$ where $c_2\in \{1,2,3,...,N\}/\mathcal{M}_1$,

\begin{align*}
    P(C_2=c_2|C_1=c_1) \propto \frac{\frac{1}{[QCID(\mathcal{M}_2)]^{\beta_2}}}{\frac{1}{[QCID(\mathcal{M}_1)]^{\beta_1}}+\alpha \sqrt{\frac{3 \ln k}{2 T_{c_1}}}}
\end{align*}

To derive the $QCID(\mathcal{M}_2)$ for each  $c_2\in \{1,2,3,...,N\}/\{\mathcal{M}_1\}$, according to Theorem $\ref{theorem1}$, we need to find the  principle submatrix of S, denoted by $S_2$, in which the set of column indices is $\mathcal{M}_2$. Then we need to calculate the sum of all the elements in $S_2$. Since there are $N-1$ different values for $c_2$, there will be $N-1$ different $S_2$. Also, because we have already calculate the sum of all the elements in $S_1$, which is a submatrix of $S_2$, in our first step, we now only need to sum over all the other elements in  $S_2$. Since the dimension of $S_2$ is $2\times 2$, we need to do the computation for $(N-1) \times (2^2-1)$ times.

This procedure goes on. After selecting $\mathcal{M}_{m-1}=\{c_1, c_2, ..., c_{m-1}\}$, where $3\leq m \leq M$, we need to select $c_m\in \{1,2,3,...,N\}/\mathcal{M}_m$ to form $\mathcal{M}_m=\mathcal{M}_{m-1} \bigcup \{c_m\}$. Before selecting the $m$-th client, we need to calculate the following value for all the  $\mathcal{M}_m=\{c_1,c_2,..., c_m\}$ where $c_m\in \{1,2,3,...,N\}/\mathcal{M}_{m-1}$,

\begin{align*}
    P(C_m=c_m|C_{m-1}=c_{m-1},...,C_2=c_2,C_1=c_1)
    \propto  \frac{[QCID(\mathcal{M}_{m-1})]^{\beta_{m-1}}}{[QCID(\mathcal{M}_m)]^{\beta_m }}
\end{align*}

To derive the $QCID(\mathcal{M}_{m})$ for each  $c_m\in \{1,2,3,...,N\}/\{\mathcal{M}_{m-1}\}$, according to Theorem $\ref{theorem1}$, we need to find the  principle submatrix of S, denoted by $S_m$, in which the set of column indices is $\mathcal{M}_m$. Then we need to calculate the sum of all the elements in $S_m$. Since there are $N-(m-1)$ different values for $c_m$, there will be $N-(m-1)$ different $S_m$. Since we have already calculate the sum of all the elements in $S_{m-1}$, which is a submatrix of $S_m$, in our previous step, now we only need to sum all the other elements in  $S_m$. Since the dimension of $S_m$ is $m\times m$, we need to conduct the computation for $(N-(m-1)) \times (m^2-(m-1)^2)$ times.

In summary, in our strategy, the total times of computations we need to conduct are 
\begin{align*}
    &N+(N-1) \times (2^2-1)+...+(N-(m-1)) \times (m^2-(m-1)^2)+...+(N-M) \times (M^2-(M-1)^2)\\
    &\leq N+N \times (2^2-1)+...+N \times (m^2-(m-1)^2)+...+ N \times (M^2-(M-1)^2)\\
    &=N\times M^2~,
\end{align*}
which finishes the proof that the computation complexity for our method is $\mathcal{O}\left(N\times M^2\right)$.
\end{proof}

\subsection{Proof of Theorem \ref{theo:conv}}

\begin{proof}
Suppose there are N available clients and their indices are denoted by $\{1,2,3,..,N\}$. Our goal is to get a subset $\mathcal{M}$ of $\{1,2,3,..,N\}$ following the probability law $S$ of some client selection strategy. Let $\boldsymbol{w}_{n}^{(k, t)}$ denote the model parameter of client $n$ after $t$ local updates in the $k$-th communication round and $\boldsymbol{w}^{(k, 0)}$ denote the global model parameter at the beginning of the $k$-th communication round. According to the proof of Theorem 1 in \cite{wang2020tackling}, we can define the following auxiliary variables for the setting where we adopt FedAvg as the FL optimizer and all the client conduct $\tau$ local updates in each communication round $k$:

Normalized Stochastic Gradient: $\quad \boldsymbol{d}_{n}^{(k)}=\frac{1}{\tau}\sum_{k=0}^{\tau-1} g_{n}\left(\boldsymbol{w}_{n}^{(k, t)}\right)$,

Normalized Gradient:  $\boldsymbol{h}_{n}^{(k)}=\frac{1}{\tau} \sum_{k=0}^{\tau-1} \nabla F_{n}\left(\boldsymbol{w}_{n}^{(k, t)}\right)$.

Normalized Class-wise Gradient:  $\boldsymbol{h}_{(n,b)}^{(k)}=\frac{1}{\tau} \sum_{k=0}^{\tau-1} \nabla F_{(n,b)}\left(\boldsymbol{w}_{n}^{(k, t)}\right)$.

It is easy to verify that $\boldsymbol{h}_{n}^{(k)}=\sum_{b=1}^B\alpha_{(n,b)}\boldsymbol{h}_{(n,b)}^{(k)}$.

According to the proof of Theorem 1 in \cite{wang2020tackling}, one can show that $\mathbb{E}\left[\boldsymbol{d}_{n}^{(k)}-\boldsymbol{h}_{n}^{(k)}\right]=0$. Besides, since clients are independent to each other, we have $\mathbb{E}\left\langle\boldsymbol{d}_{n}^{(k)}-\boldsymbol{h}_{n}^{(k)}, \boldsymbol{d}_{n^\prime}^{(k)}-\boldsymbol{h}_{n^\prime}^{(k)}\right\rangle=0, \forall n \neq n^\prime.$ Recall that the update rule of the global model can be written as follows:
$$
\boldsymbol{w}^{(k+1, 0)}-\boldsymbol{w}^{(k, 0)}=-\eta \frac{\sum_{n\in \mathcal{M} }q_n \boldsymbol{d}_{n}^{(k)}}{\sum_{n\in \mathcal{M} } q_n},
$$
where $\eta$ is the learning rate.
According to the Lipschitz-smooth assumption for the global objective function $\widetilde{F}$ (Asssumption \ref{ass1}), it follows that 
\begin{align}
& \mathbb{E}\left[\widetilde{F}\left(\boldsymbol{w}^{(k+1, 0)}\right)\right]-\widetilde{F}\left(\boldsymbol{w}^{(k, 0)}\right) \nonumber\\
\leq &- \eta \underbrace{\mathbb{E}\left[\left\langle\nabla \widetilde{F}\left(\boldsymbol{w}^{(k, 0)}\right), \frac{\sum_{n\in \mathcal{M} }q_n \boldsymbol{d}_{n}^{(k)}}{\sum_{n\in \mathcal{M} } q_n}\right\rangle\right]}_{T_{1}}+\frac{\eta^{2} L_{\widetilde{F}}}{2} \underbrace{\mathbb{E}\left[\left\| \frac{\sum_{n\in \mathcal{M} }q_n \boldsymbol{d}_{n}^{(k)}}{\sum_{n\in \mathcal{M} } q_n}\right\|^{2}\right]}_{T_{2}}
\label{proof:T1T2}
\end{align}
where the expectation is taken over randomly selected indices set $\mathcal{M}$ as well as mini-batches $\xi_{i}^{(k, t)}, \forall n \in$ $\{1,2, \ldots, m\}, t \in\left\{0,1, \ldots, \tau-1\right\}$

Similar to the proof in \cite{wang2020tackling}, to bound the $T_1$ in (\ref{proof:T1T2}), we should notice that
\begin{align}
T_{1} &=\mathbb{E}\left[\left\langle\nabla \widetilde{F}\left(\boldsymbol{w}^{(k, 0)}\right), \frac{\sum_{n\in \mathcal{M} }q_n \left(\boldsymbol{d}_{n}^{(k)}-\boldsymbol{h}_{n}^{(k)}\right)}{\sum_{n\in \mathcal{M} } q_n}\right\rangle\right]+\mathbb{E}\left[\left\langle\nabla \widetilde{F}\left(\boldsymbol{w}^{(k, 0)}\right), \frac{\sum_{n\in \mathcal{M} }q_n \boldsymbol{h}_{n}^{(k)}}{\sum_{n\in \mathcal{M} } q_n}\right\rangle\right] \nonumber \\
&=\mathbb{E}\left[\left\langle\nabla \widetilde{F}\left(\boldsymbol{w}^{(k, 0)}\right), \frac{\sum_{n\in \mathcal{M} }q_n \boldsymbol{h}_{n}^{(k)}}{\sum_{n\in \mathcal{M} } q_n}\right\rangle\right] \nonumber\\
&=\frac{1}{2}\left\|\nabla \widetilde{F}\left(\boldsymbol{w}^{(k,0)}\right)\right\|^{2}+\frac{1}{2} \mathbb{E}\left[\left\|\frac{\sum_{n\in \mathcal{M} }q_n \boldsymbol{h}_{n}^{(k)}}{\sum_{n\in \mathcal{M} } q_n}\right\|^{2}\right]-\frac{1}{2} \mathbb{E}\left[\left\|\nabla \widetilde{F}\left(\boldsymbol{w}^{(k, 0)}\right)-\frac{\sum_{n\in \mathcal{M} }q_n \boldsymbol{h}_{n}^{(k)}}{\sum_{n\in \mathcal{M} } q_n}\right\|^{2}\right]
\label{proofT1}
\end{align}
where the last equation uses the fact: $2\langle a, b\rangle=\|a\|^{2}+\|b\|^{2}-\|a-b\|^{2}$.

$T_2$ is similar as the one in \cite{wang2020tackling}. According to the proof in Section C.3 of \cite{wang2020tackling} , we have the following bound for $T_2$, 
\begin{align} 
T_{2}\leq &2 \sigma^{2}\mathbb{E} \frac{\sum_{n\in \mathcal{M} }q^2_n }{(\sum_{n\in \mathcal{M} } q_n)^2}+2 \mathbb{E}\left[\left\|\frac{\sum_{n\in \mathcal{M} }q_n \boldsymbol{h}_{n}^{(k)}}{\sum_{n\in \mathcal{M} } q_n}\right\|^{2}\right] \nonumber \\
\leq &2\sigma^{2}+2 \mathbb{E}\left[\left\|\frac{\sum_{n\in \mathcal{M} }q_n \boldsymbol{h}_{n}^{(k)}}{\sum_{n\in \mathcal{M} } q_n}\right\|^{2}\right]
\label{proofT2}
\end{align}
Plugging (\ref{proofT1}) and (\ref{proofT2}) back into (\ref{proof:T1T2}), we have 

\begin{align}
    & \mathbb{E}\left[\widetilde{F}\left(\boldsymbol{w}^{(k+1, 0)}\right)\right]-\widetilde{F}\left(\boldsymbol{w}^{(k, 0)}\right) \nonumber\\
\leq &- \eta \underbrace{\mathbb{E}\left[\left\langle\nabla \widetilde{F}\left(\boldsymbol{w}^{(k, 0)}\right), \frac{\sum_{n\in \mathcal{M} }q_n \boldsymbol{d}_{n}^{(k)}}{\sum_{n\in \mathcal{M} } q_n}\right\rangle\right]}_{T_{1}}+\frac{\eta^{2} L_{\widetilde{F}}}{2} \underbrace{\mathbb{E}\left[\left\| \frac{\sum_{n\in \mathcal{M} }q_n \boldsymbol{d}_{n}^{(k)}}{\sum_{n\in \mathcal{M} } q_n}\right\|^{2}\right]}_{T_{2}}\nonumber\\
\leq& -\frac{1}{2}\eta\left\|\nabla \widetilde{F}\left(\boldsymbol{w}^{(k,0)}\right)\right\|^{2}-\frac{1}{2} \eta\mathbb{E}\left[\left\|\frac{\sum_{n\in \mathcal{M} }q_n \boldsymbol{h}_{n}^{(k)}}{\sum_{n\in \mathcal{M} } q_n}\right\|^{2}\right]+\frac{1}{2}\eta \mathbb{E}\left[\left\|\nabla \widetilde{F}\left(\boldsymbol{w}^{(k, 0)}\right)-\frac{\sum_{n\in \mathcal{M} }q_n \boldsymbol{h}_{n}^{(k)}}{\sum_{n\in \mathcal{M} }q_n}\right\|^{2}\right]\nonumber\\
+&\eta^2L_{\widetilde{F}}\sigma^{2}+\eta^2L_{\widetilde{F}}\mathbb{E}\left[\left\|\frac{\sum_{n\in \mathcal{M} }q_n \boldsymbol{h}_{n}^{(k)}}{\sum_{n\in \mathcal{M} } q_n}\right\|^{2}\right]
\end{align}
If we set $\eta\leq \frac{1}{2L}$, we have 
\begin{align}
    & \mathbb{E}\left[\widetilde{F}\left(\boldsymbol{w}^{(k+1, 0)}\right)\right]-\widetilde{F}\left(\boldsymbol{w}^{(k, 0)}\right) \nonumber\\
\leq& -\frac{1}{2}\eta\left\|\nabla \widetilde{F}\left(\boldsymbol{w}^{(k,0)}\right)\right\|^{2}+\frac{1}{2}\eta \mathbb{E}\left[\left\|\nabla \widetilde{F}\left(\boldsymbol{w}^{(k, 0)}\right)-\frac{\sum_{n\in \mathcal{M} }q_n \boldsymbol{h}_{n}^{(k)}}{\sum_{n\in \mathcal{M} }q_n}\right\|^{2}\right]+\eta^2L_{\widetilde{F}}\sigma^{2}.
\label{proof:middle}
\end{align}
Now we focus on the $\mathbb{E}\left[\left\|\nabla \widetilde{F}\left(\boldsymbol{w}^{(k, 0)}\right)-\frac{\sum_{n\in \mathcal{M} }q_n \boldsymbol{h}_{n}^{(k)}}{\sum_{n\in \mathcal{M} }q_n}\right\|^{2}\right]$ in the following:
\begin{align}
    &\mathbb{E}\left\|\nabla  \widetilde{F}\left(\boldsymbol{w}^{(k, 0)}\right)-\frac{\sum_{n\in \mathcal{M} }q_n \boldsymbol{h}_{n}^{(k)}}{\sum_{n\in \mathcal{M} }q_n}\right\|^{2}=\mathbb{E}\left\|\frac{1}{B}\sum_{b=1}^B \nabla\widetilde{F}_b\left(\boldsymbol{w}^{(k, 0)}\right)-\frac{\sum_{n\in \mathcal{M} }q_n\left(\sum_{b=1}^B\alpha_{(n,b)}\boldsymbol{h}_{(n,b)}^{(k)}\right)}{\sum_{n\in \mathcal{M} }q_n}\right\|^2\nonumber\\
    &=\mathbb{E}\left\|\frac{1}{B}\sum_{b=1}^B \nabla\widetilde{F}_b\left(\boldsymbol{w}^{(k, 0)}\right)-\sum_{b=1}^B\frac{\sum_{n\in \mathcal{M} }q_n\alpha_{(n,b)}\boldsymbol{h}_{(n,b)}^{(k)}}{\sum_{n\in \mathcal{M} }q_n}\right\|^2\nonumber\\
    &\leq2\mathbb{E}\left\|\frac{1}{B}\sum_{b=1}^B \nabla\widetilde{F}_b\left(\boldsymbol{w}^{(k, 0)}\right)-\sum_{b=1}^B\frac{\sum_{n\in \mathcal{M} }q_n\alpha_{(n,b)}\nabla\widetilde{F}_b\left(\boldsymbol{w}^{(k, 0)}\right)}{\sum_{n\in \mathcal{M} }q_n}\right\|^2\nonumber\\
    &+2\mathbb{E}\left\|\sum_{b=1}^B\frac{\sum_{n\in \mathcal{M} }q_n\alpha_{(n,b)}\nabla\widetilde{F}_b\left(\boldsymbol{w}^{(k, 0)}\right)}{\sum_{n\in \mathcal{M} }q_n}-\sum_{b=1}^B\frac{\sum_{n\in \mathcal{M} }q_n\alpha_{(n,b)}\boldsymbol{h}_{(n,b)}^{(k)}}{\sum_{n\in \mathcal{M} }q_n}\right\|^2\nonumber\\
    &=2\underbrace{\mathbb{E}\left\|\sum_{b=1}^B (\frac{1}{B}-\frac{\sum_{n\in \mathcal{M} }q_n\alpha_{(n,b)}}{\sum_{n\in \mathcal{M} }q_n})\nabla\widetilde{F}_b\left(\boldsymbol{w}^{(k, 0)}\right)\right\|^2}_{T_3}+2\underbrace{\mathbb{E}\left\|\sum_{b=1}^B\frac{\sum_{n\in \mathcal{M} }q_n\alpha_{(n,b)}[\nabla\widetilde{F}_b\left(\boldsymbol{w}^{(k, 0)}\right)-\boldsymbol{h}_{(n,b)}^{(k)}]}{\sum_{n\in \mathcal{M} }q_n} \right\|^2}_{T_4}
    \label{proof:T3T4}
\end{align}
For $T_3$, according to the Cauchy-Schwarz inequality and Assumption \ref{ass3}, we have 
\begin{align}
    \mathbb{E}\left\|\sum_{b=1}^B (\frac{1}{B}-\frac{\sum_{n\in \mathcal{M} }q_n\alpha_{(n,b)}}{\sum_{n\in \mathcal{M} }q_n})\nabla\widetilde{F}_b\left(\boldsymbol{w}^{(k, 0)}\right)\right\|^2 &\leq  B\mathbb{E}\left[\sum_{b=1}^B (\frac{1}{B}-\frac{\sum_{n\in \mathcal{M} }q_n\alpha_{(n,b)}}{\sum_{n\in \mathcal{M} }q_n})^2\left\|\frac{1}{B}\sum_{b=1}^B \nabla\widetilde{F}_b\left(\boldsymbol{w}^{(k, 0)}\right)\right\|^2\right]\nonumber\\
    &=B\|\nabla\widetilde{F}\left(\boldsymbol{w}^{(k, 0)}\right)\|^2\mathbb{E}[QCID(\mathcal{M})]+\gamma^2\mathbb{E}[QCID(\mathcal{M})]
    \label{proof:qcid}
\end{align}
For $T_4$, we have 
\begin{align}
    &\mathbb{E}\left\|\sum_{b=1}^B\frac{\sum_{n\in \mathcal{M} }q_n\alpha_{(n,b)}[\nabla\widetilde{F}_b\left(\boldsymbol{w}^{(k, 0)}\right)-\boldsymbol{h}_{(n,b)}^{(k)}]}{\sum_{n\in \mathcal{M} }q_n} \right\|^2\nonumber\\
    &=\mathbb{E}\left\|\sum_{b=1}^B\frac{\sum_{n\in \mathcal{M} }q_n\alpha_{(n,b)}[\nabla\widetilde{F}_b\left(\boldsymbol{w}^{(k, 0)}\right)-\nabla F_{(n,b)}\left(\boldsymbol{w}^{(k, 0)}\right)+\nabla F_{(n,b)}\left(\boldsymbol{w}^{(k, 0)}\right)-\boldsymbol{h}_{(n,b)}^{(k)}]}{\sum_{n\in \mathcal{M} }q_n} \right\|^2\nonumber\\
    &\leq 2\mathbb{E}\left\|\sum_{b=1}^B\frac{\sum_{n\in \mathcal{M} }q_n\alpha_{(n,b)}[\nabla\widetilde{F}_b\left(\boldsymbol{w}^{(k, 0)}\right)-\nabla F_{(n,b)}\left(\boldsymbol{w}^{(k, 0)}\right)]}{\sum_{n\in \mathcal{M} }q_n} \right\|^2\\
    &+2\mathbb{E}\left\|\sum_{b=1}^B\frac{\sum_{n\in \mathcal{M} }q_n\alpha_{(n,b)}[\nabla F_{(n,b)}\left(\boldsymbol{w}^{(k, 0)}\right)-\boldsymbol{h}_{(n,b)}^{(k)}
]}{\sum_{n\in \mathcal{M} }q_n} \right\|^2\nonumber\\
    &\leq 2\kappa^2 \sum_{b=1}^B\frac{\sum_{n\in \mathcal{M} }q_n\alpha_{(n,b)}}{\sum_{n\in \mathcal{M} }q_n}+2\left\|\sum_{b=1}^B\frac{\sum_{n\in \mathcal{M} }q_n\alpha_{(n,b)}[\nabla F_{(n,b)}\left(\boldsymbol{w}^{(k, 0)}\right)-\boldsymbol{h}_{(n,b)}^{(k)}
]}{\sum_{n\in \mathcal{M} }q_n} \right\|^2\nonumber\\
&\leq 2\kappa^2 +2\left\|\sum_{b=1}^B\frac{\sum_{n\in \mathcal{M} }q_n\alpha_{(n,b)}[\nabla F_{(n,b)}\left(\boldsymbol{w}^{(k, 0)}\right)-\boldsymbol{h}_{(n,b)}^{(k)}
]}{\sum_{n\in \mathcal{M} }q_n} \right\|^2,
\label{proof:last-2}
\end{align}
where where $\kappa=max_{\{n,b\}}\kappa_{\{n,b\}}$. According to the results from the proof in C.5 in \cite{wang2020tackling}, we have 
\begin{align}
    &\mathbb{E}\left\|\nabla F_{(n,b)}\left(\boldsymbol{w}^{(k, 0)}\right)-\boldsymbol{h}_{(n,b)}^{(k)}\right\|^2 \leq \frac{L_{n,b}^{2}}{\tau} \sum_{k=0}^{\tau-1}  \mathbb{E}\left[\left\|\boldsymbol{w}^{(k, 0)}-\boldsymbol{w}_{n}^{(k, t)}\right\|^{2}\right] \nonumber\\
    &\leq \frac{L^{2}}{\tau} \sum_{k=0}^{\tau-1}  \mathbb{E}\left[\left\|\boldsymbol{w}^{(k, 0)}-\boldsymbol{w}_{n}^{(k, t)}\right\|^{2}\right] \nonumber\\
    &\leq \frac{2 \eta^{2} L^{2} \sigma^{2}}{1-D}\left(\tau-1\right)+\frac{D}{1-D} \mathbb{E}\left[\left\|\nabla F_{i}\left(\boldsymbol{w}^{(k, 0)}\right)\right\|^{2}\right]\nonumber\\
    &\leq \frac{2 \eta^{2} L^{2} \sigma^{2}}{1-D}\left(\tau-1\right)+\frac{2D}{1-D} \left\|\nabla  \widetilde{F}\left(\boldsymbol{w}^{(k, 0)}\right)\right\|^{2}+\frac{2D}{1-D}\mathbb{E}\left[\left\|\nabla F_{i}\left(\boldsymbol{w}^{(k, 0)}\right)-\nabla  \widetilde{F}\left(\boldsymbol{w}^{(k, 0)}\right)\right\|^{2}\right]\nonumber\\
    &\leq \frac{2 \eta^{2} L^{2} \sigma^{2}}{1-D}\left(\tau-1\right)+\frac{2D}{1-D} \left\|\nabla  \widetilde{F}\left(\boldsymbol{w}^{(k, 0)}\right)\right\|^{2}+\frac{2D}{1-D}\mathbb{E}\left[\left\|\nabla F_{i}\left(\boldsymbol{w}^{(k, 0)}\right)-\nabla  \widetilde{F}\left(\boldsymbol{w}^{(k, 0)}\right)\right\|^{2}\right]\nonumber\\
    &\leq \frac{2 \eta^{2} L^{2} \sigma^{2}}{1-D}\left(\tau-1\right)+\frac{2D}{1-D} \left\|\nabla  \widetilde{F}\left(\boldsymbol{w}^{(k, 0)}\right)\right\|^{2}+\frac{1}{B}\frac{2D}{1-D}\kappa^2
    \label{proof:boundT4sub}
\end{align}
where $L=max_{\{n,b\}}L_{n,b}$ and $D=4\eta^2L^2\tau(\tau-1)$.

Combining the results in  (\ref{proof:middle}), (\ref{proof:T3T4}), (\ref{proof:qcid}), (\ref{proof:last-2}) and (\ref{proof:boundT4sub}), it is easy to derive that 
\begin{align}
    &\mathbb{E}\left[\widetilde{F}\left(\boldsymbol{w}^{(k+1, 0)}\right)\right]-\widetilde{F}\left(\boldsymbol{w}^{(k, 0)}\right) \nonumber\\
\leq& -(\frac{1}{2}-B\delta\mathbb{E}[QCID]-\frac{4D}{1-D})\eta\left\|\nabla  \widetilde{F}\left(\boldsymbol{w}^{(k, 0)}\right)\right\|^{2}+4\eta\kappa^2+\frac{4 \eta^{3} L^{2} \sigma^{2}}{1-D}\left(\tau-1\right)\nonumber\\
&+\frac{1}{B}\frac{4D}{1-D}\eta\kappa^2+\eta^2L_{\widetilde{F}}\sigma^{2}+\gamma^2\eta\mathbb{E}[QCID]
\end{align}

Now we have
\begin{align}
    &\frac{\mathbb{E}\left[\widetilde{F}\left(\boldsymbol{w}^{(k+1, 0)}\right)\right]-\widetilde{F}\left(\boldsymbol{w}^{(k, 0)}\right)}{\eta}\nonumber\\
\leq& -(\frac{1}{2}-B\delta\mathbb{E}[QCID]-\frac{4D}{1-D})\left\|\nabla  \widetilde{F}\left(\boldsymbol{w}^{(k, 0)}\right)\right\|^{2}+4\kappa^2\nonumber \\
&+\frac{4 \eta^{2} L^{2} \sigma^{2}}{1-D}\left(\tau-1\right)+\frac{1}{B}\frac{4D}{1-D}\kappa^2+\gamma^2\mathbb{E}[QCID]+\eta L_{\widetilde{F}}\sigma^{2}
\end{align}

Taking the total expectation and averaging over all rounds, one can obtain
\begin{align}
    &\frac{\mathbb{E}\left[\widetilde{F}\left(\boldsymbol{w}^{(K,0)}\right)\right]-\widetilde{F}\left(\boldsymbol{x}^{(0, 0)}\right)}{\eta K}\leq-(\frac{1}{2}-B\delta\mathbb{E}[QCID]-\frac{4D}{1-D})\frac{1}{K}\sum_{t=1}^{K-1}\mathbb{E}\left\|\nabla  \widetilde{F}\left(\boldsymbol{w}^{(k, 0)}\right)\right\|^{2} \nonumber\\
& \gamma^2\mathbb{E}[QCID]+(4+\frac{1}{B}\frac{4D}{1-D})\kappa^2+\frac{4 \eta^{2} L^{2} \sigma^{2}}{1-D}\left(\tau-1\right)+\eta L\sigma^{2}
\end{align}

Finally, we have
\begin{align}
\min _{k \leq K}\left\|\nabla F\left(\boldsymbol{w}^{(k, 0)}\right)\right\|^{2}\leq
    \frac{1}{K}\sum_{k=1}^{K-1}\mathbb{E}\left\|\nabla  \widetilde{F}\left(\boldsymbol{w}^{(k, 0)}\right)\right\|^{2}&\leq \frac{1}{(\frac{1}{2}-B\delta\mathbb{E}[QCID(\mathcal{M})]-\frac{4D}{1-D})}[\frac{\widetilde{F}\left(\boldsymbol{w}^{(0, 0)}\right)-\widetilde{F}_{min}}{\eta K}\nonumber\\
    +(4+\frac{1}{B}\frac{4D}{1-D})\kappa^2
    &+\frac{4 \eta^{2} L^{2} \sigma^{2}}{1-D}\left(\tau-1\right)+\eta L_{\widetilde{F}}\sigma^{2}+\gamma^2\mathbb{E}[QCID]]
\end{align}

If setting $\eta=\frac{s}{10L\sqrt{\tau(\tau-1)K}}$ with $s<1$, we have 
\begin{align}
    \min _{k \leq K}\left\|\nabla F\left(\boldsymbol{w}^{(k, 0)}\right)\right\|^{2}&\leq \frac{1}{\frac{1}{3}-B\delta\mathbb{E}[QCID(\mathcal{M})]}[\frac{\widetilde{F}\left(\boldsymbol{w}^{(0, 0)}\right)-\widetilde{F}_{min}}{s \sqrt{K}/(10L\sqrt{\tau(\tau-1)})}\nonumber\\
    &+5\kappa^2+\frac{\sigma^2s^2}{25\tau K}+\frac{sL_{\widetilde{F}}\sigma^2}{10L\sqrt{\tau(\tau-1)K}}+\gamma^2\mathbb{E}[QCID]]
\end{align}
Since $\widetilde{F}$ is larger than 0, $F_{min}>0$. Now we let $\boldsymbol{w}^{(k)}$ denote the global model parameter at the $k$-th communication round and $\boldsymbol{w}^{(0)}$ denote the initial parameter. After changing the notations, we can finish our proof by the following:
\begin{align*}
    &\min _{k \leq K}\left\|\nabla \widetilde{F}\left(\boldsymbol{w}^{(k)}\right)\right\|^{2} \leq \frac{1}{\frac{1}{3}-B\delta\mathbb{E}[QCID(\mathcal{M})]}[\frac{\sigma^2s^2}{25\tau K}+\frac{sL_{\widetilde{F}}\sigma^2}{10L\sqrt{\tau(\tau-1)K}}\nonumber\\
    &+5\kappa^2+\frac{10L\sqrt{\tau(\tau-1)}\widetilde{F}\left(\boldsymbol{w}^{(0)}\right)}{s \sqrt{K}}+\gamma^2\mathbb{E}[QCID]],
\end{align*}

\end{proof}

\section{Supplemental Experiment Settings and Results}

\subsection{The Experimental Settings in Section \ref{CImotivation}}\label{setting2.1}
 We adopt an MLP model with one hidden layer of 64 units and FedAvg \cite{mcmahan2016communication} as the FL optimizer. 
 In Figure \ref{fig:fig1-sub-first}, we allocate the MNIST data to $N = 100$ clients with each client only accessing to the same amount of data from one class. In Figure \ref{fig:fig1-sub-second}, each client is associated with the same amount of data from two classes. In Figure \ref{fig:fig1-sub-third} and \ref{fig:fig1-sub-fourth}, we first allocate the whole MNIST dataset to $N = 200$ clients and pick 100 to construct a class-imbalanced global dataset. The global dataset with the 100 clients has the same amount of $n_1$ data samples for five classes and has the same amount of $n_2$ data samples for the other five classes. The ration r between $n_1$ and $n_2$ is set to $3:1$.
 
  In each training round (communication round), all of the clients conduct 5 local training epochs. The batch size is $50$ for each client. The local optimizer is SGD with a weight decay of 0.0005. The learning rate is  0.01 initially and the decay factor is 0.9992. We terminate the FL training after 200 training rounds (communication rounds) and then evaluate the model's performance on the test dataset of MNIST. 
 
 \subsection{Additional Experimental Settings in Section \ref{Sec:experiments} }\label{Sec:settingsec5}
 
The model we adopt has two convolutional layers with the number of kernels being 6 and 16, respectively. And all convolution kernels are of size 5 × 5. The outputs of convolutional layers are fed into two hidden layers  with 120 and 84 units. 

In our implementation of Power-of-choice selection strategy (pow-d)\cite{Cho2020ClientSI}, we first sample a candidate set $\mathcal{A}$ of 20 clients without replacement such that client $n$ is chosen with probability proportional to the size of their local dataset $q_n$. Then the server sends the current global model to the clients in set $\mathcal{A}$, and these clients compute and send back to the server their local loss. To derive $\mathcal{M}$, we select M clients who have the highest loss from  $\mathcal{A}$.

In our implementation of the method in \cite{yang2020federated} (Fed-cucb), the exploration factor to balance the trade-off between exploitation and exploration is set as  0.2 and the forgetting factor  as 0.99, which is the same as the settings in \cite{yang2020federated}.

With the help of FHE, we can derive the matrix of inner products $S$ accurately. Hence, in the simulation of our method, Fed-CBS, we ignore the process of deriving $S$ and focus on our sampling strategy.

\subsection{Additional Details for the Experimental Settings in Case 1 and Case 2}\label{Sec:settingCase12}

\paragraph{Case 1} In this setting, we have 120 clients in total, and each client has only one class of data. 

When $n_1:n_2=3:1$, there are 18 clients having the data from the 1st class, 18 clients having the data from the 2nd class, 18 clients having the data from the 3rd class, 18 clients having the data from the 4th class, and 18 clients having the data from the 5th class. There are 6 clients with data from the 6th class,  6 clients with data from the 7th class, 6 clients with data from the 8th class, 6 clients with data from the 9th class, and 6 clients the data from the 10th class.

When $n_1:n_2=5:1$, there are 20 clients having the data from the 1st class, 20 clients having the data from the 2nd class, 20 clients having the data from the 3rd class, 20 clients having the data from the 4th class and 20 clients having the data from the 5th class. There are 4 clients having the data from the 6th class,  4 clients having data from the 7th class, 4 clients having data from the 8th class, 4 clients having data from the 9th class, and 4 clients having the data from the 10th class.

Then we uniformly set 30$\%$ (36 clients) of them available. Since there are more clients which contain the data from the first 5 classes among the above 120 clients. The global dataset of these 36 clients is often class-imbalanced.

\paragraph{Case 2} In this setting, we have 200 clients in total and each client has only one class of data. For all the $i\in \{1,2,...,10\}$, there are 20 clients having the data from  the $i-$th class.

When $n_1:n_2=3:1$, we randomly pick 9 clients from the 20 clients which have the data from the 1st class and set them available. We randomly pick 9 clients from the 20 clients which have the data from the 2nd class and set them available. Similarly, for the $k$-th class ($2 < k \leq 5$), we randomly pick 9 clients from the 20 clients which have the data from the $k$-th class and set them available. On the contrary, we randomly pick 3 clients from the 20 clients which have the data from the 6th class and set them available. We randomly pick 3 clients from the 20 clients which have the data from the 7th class and set them available. Similarly, for $7 < k \leq 10$, we randomly pick 3 clients from the 20 clients which have the data from the $k$-th class and set them available. There are 60 clients in total.

When $n_1:n_2=5:1$, we randomly pick 10 clients from the 20 clients which have the data from the 1st class and set them available. We randomly pick 10 clients from the 20 clients which have the data from the 2nd class and set them available. For the $k$-tth class ($2 < k \leq 5)$, we randomly pick 10 clients from the 20 clients which have the data from the $k$-th class and set them available. On the contrary, we randomly pick 2 clients from the 20 clients which have the data from the 6th class and set them available. We randomly pick 2 clients from the 20 clients which have the data from the 7th class and set them available. And for the other $k$-th class ($7 < k \leq 10$), we randomly pick 2 clients from the 20 clients which have the data from the $k$-th class and set them available. There are 60 clients in total.

Since there are more clients that contain the data from the first 5 classes among the above 60 clients, the global dataset of these 60 clients is always class-imbalanced. 

The difference between the settings of Case 1 and Case 2 is that we uniformly set 30$\%$ clients available in Case 1 but non-uniformly set 30$\%$ clients available in Case 2. Nevertheless, the global datasets of the available clients are both class-imbalanced in both cases.

\subsection{The Averaged QCID Values for Case 1 and Case 2 in Section \ref{RCIGD}}\label{Sec:QCIDcase12}
\begin{table*}[!hbtp]
\centering
\renewcommand{\arraystretch}{0.9}
\setlength\tabcolsep{15pt}
\scalebox{1}{
\begin{tabular}{|c|c|c|c|c|c|c|}
\hline
\multicolumn{2}{|l|}{$\mathbb{E}[QCID](10^{-2})$}                                                                          & all            & rand        & pow-d            & Fed-cucb                     & Fed-CBS                          \\ \hline
\multirow{2}{*}{Case 1} & 3:1 & 2.90$\pm$0.02  & 9.33$\pm$0.17  & 13.70$\pm$0.39 & 1.39$\pm$0.37 & \textbf{0.57$\pm$0.04}                 \\ \cline{2-7} & 5:1 & 6.17$\pm$0.04  & 12.36$\pm$0.20 & 16.63$\pm$0.74 &   3.43$\pm$0.76              &        \textbf{2.41$\pm$0.07}          \\ \hline
\multirow{2}{*}{Case 2} & 3:1 & 2.50$\pm$0.00     & 9.91$\pm$0.16  & 13.68$\pm$0.72 & 1.89$\pm$1.72  & \textbf{0.001$\pm$0.001} \\ \cline{2-7} & 5:1 & 4.44$\pm$0.00     & 11.70$\pm$0.20 & 15.68$\pm$0.96 & 2.63$\pm$2.40  & \textbf{0.002$\pm$0.001} \\ \hline
\end{tabular}}
\caption{The averaged $QCID$ values for four baselines and our method. Our method, Fed-CBS, has successfully reduced the class-imbalance. Since the global dataset of all the 60 available clients is always class-imbalanced and the ratio is always fixed in case 2, the $QCID$ value is fixed and the derivation of it is always zero.}
\label{app:tableqcid}
\end{table*}

\subsection{Experiment Results of Fashion-MNIST Dataset}\label{Sec:resultfmnist}
\paragraph{Experiment Setup}  We adopt an MLP model with one hidden layer of 64 units and  and FedNova \citep{wang2020tackling} as the FL optimizer . Similar to the setup in the experiment of CIFAR-10, the batch size is $50$ for each client. In each communication round, all of them conduct the same number of local updates, which allows the client with the largest local dataset to conduct 5 local training epochs. In our method, we set the $\beta_m=m$, $\gamma=10$ and $L_{b}=10^{-20}$. The local optimizer is SGD with a weight decay of 0.0005. The learning rate is  0.01 initially and the decay factor is 0.9992. We terminate the FL training after 3000 communication rounds and then evaluate the model's performance on the test dataset of Fashion-MNIST.
\begin{table*}[!hbtp]
\centering
\renewcommand{\arraystretch}{1}
\setlength\tabcolsep{9pt}
\scalebox{1}{\begin{tabular}{|c|c|c|c|c|c|c|}
\hline
\multicolumn{2}{|l|}{}                                                                                  & all            & rand        & pow-d            & Fed-cucb      & Fed-CBS          \\ \hline
\multirow{3}{*}{Communication Rounds}                                                                  & $\alpha$=0.1 & 115$\pm$17   & 185$\pm$27 & 135$\pm$22    & 124$\pm$37  & \textbf{92$\pm$6}   \\ \cline{2-7} 
                                                                                         & $\alpha$=0.2 & 173$\pm$45     & 284$\pm$54   & 218$\pm$55    & 216$\pm$24          & \textbf{166$\pm$36}   \\ \cline{2-7} 
                                                                                         & $\alpha$=0.5 & 258$\pm$44     & 331$\pm$55   & 281$\pm$54   & 284$\pm$51           & \textbf{218$\pm$36}   \\ \hline
\multirow{3}{*}{$\mathbb{E}[QCID]( 10^{-2}$)} & $\alpha$=0.1 & 1.40$\pm$0.11 & 8.20$\pm$0.19  & 11.72$\pm$0.33 & 4.24$\pm$0.59  & \textbf{0.15$\pm$0.02} \\ \cline{2-7} 
                                                                                         & $\alpha$=0.2 & 1.39$\pm$0.22  & 7.67$\pm$0.26 & 10.31$\pm$0.24  & 4.43$\pm$0.38 & \textbf{0.21$\pm$0.01} \\ \cline{2-7} 
                                                                                         & $\alpha$=0.5 & 0.94$\pm$0.07  & 5.93$\pm$0.26 & 7.68$\pm$0.28  & 4.34$\pm$0.85 & \textbf{0.22$\pm$0.01} \\ \hline
\end{tabular}}
\caption{The communication rounds required for targeted test accuracy and the averaged QCID values on Fashion-MNIST dataset. The targeted test accuracy is $78\%$ for $\alpha=0.1$, $80\%$ for $\alpha=0.2$ and $82\%$ for $\alpha=0.5$. The results are the mean and the standard deviation over 4 different random seeds.}
\label{roundqcid1}
\end{table*}

\begin{figure*}[!hbtp]
  \centering
  \includegraphics[width=1\linewidth]{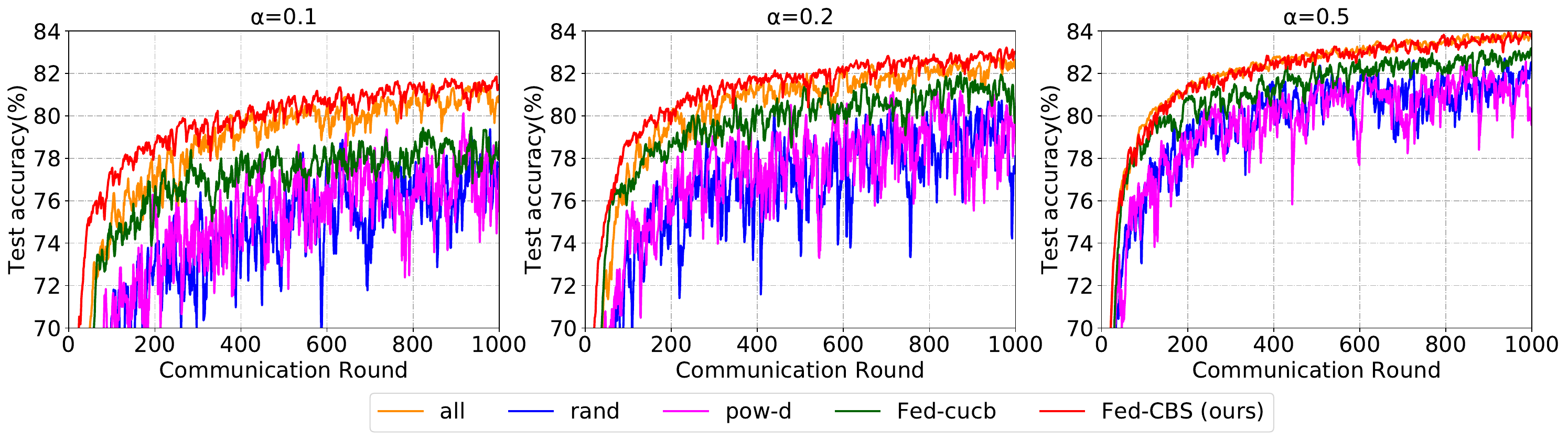}
\caption{Test accuracy on Fashion-MNIST dataset under three heterogeneous settings.}
\label{appfig:figfmnist1}
\end{figure*}
\subsubsection{Results for Class-Balanced Global Dataset}
Similar to the experiment settings, in this experiment, we set 200 clients in total with a class-balanced global dataset. The non-IID data partition among clients is based on the settings of Dirichlet distribution parameterized by the concentration parameter $\alpha$ in \cite{Hsu2019MeasuringTE}. In each communication round, we uniformly and randomly set 30$\%$ of them (i.e., 60 clients) available and select 10 clients from those 60 available ones to participate in the training. 

As shown in Table \ref{appfig:figfmnist1}, our method successfully reduces the class-imbalance, since it achieves the lowest $QCID$ value compared with other client selection strategies. Our method outperforms the other three baseline methods and achieves comparable performance in the ideal setting where all the available clients are engaged in the training. As shown in Table \ref{app:tableqcid} and Figure \ref{appfig:figfmnist1}, our method can achieve faster and more stable convergence. It is worth noting that due to the inaccurate estimation of distribution and the weakness of the greedy method discussed in Section \ref{sec:relatedwork}, the performance of Fed-cucb is much worse than ours.

\subsubsection{Results for Class-Imbalanced Global Dataset: Case 1}
Similar to the settings for Cifar-10, there are 120 clients in total and each client only has one class of data with the same quantity. The global dataset of these 120 clients is always class-imbalanced. To measure the degree of class imbalance, we let the global dataset have the same amount $n_1$ of data samples for five classes and have the same amount $n_2$ of data samples for the other five classes. The ratio $r$ between $n_1$ and $n_2$ is set to $3:1$ and $5:1$ respectively in the experiments. In each communication round, we randomly set 30$\%$ of them (\textit{i.e.}, 36 clients) available and select 10 clients to participate in the training.

As shown in the Table \ref{apptable:imbalancebest} and Figure \ref{appfig:case1}, our method can achieve faster and more stable convergence, and even better performance than the ideal setting where all the available clients are engaged. The performance of Fed-cucb \cite{yang2020federated} is better than the results on class-balanced global dataset, which is partly due to the simplicity of each client's local dataset composition in our experiments as discussed in the experiments of Cifar-10.

\subsubsection{Results for Class-Imbalanced Global Dataset: Case 2}

\begin{figure*}[ht]
\centering
\begin{subfigure}{.49\textwidth}
  \centering
  \includegraphics[width=\linewidth]{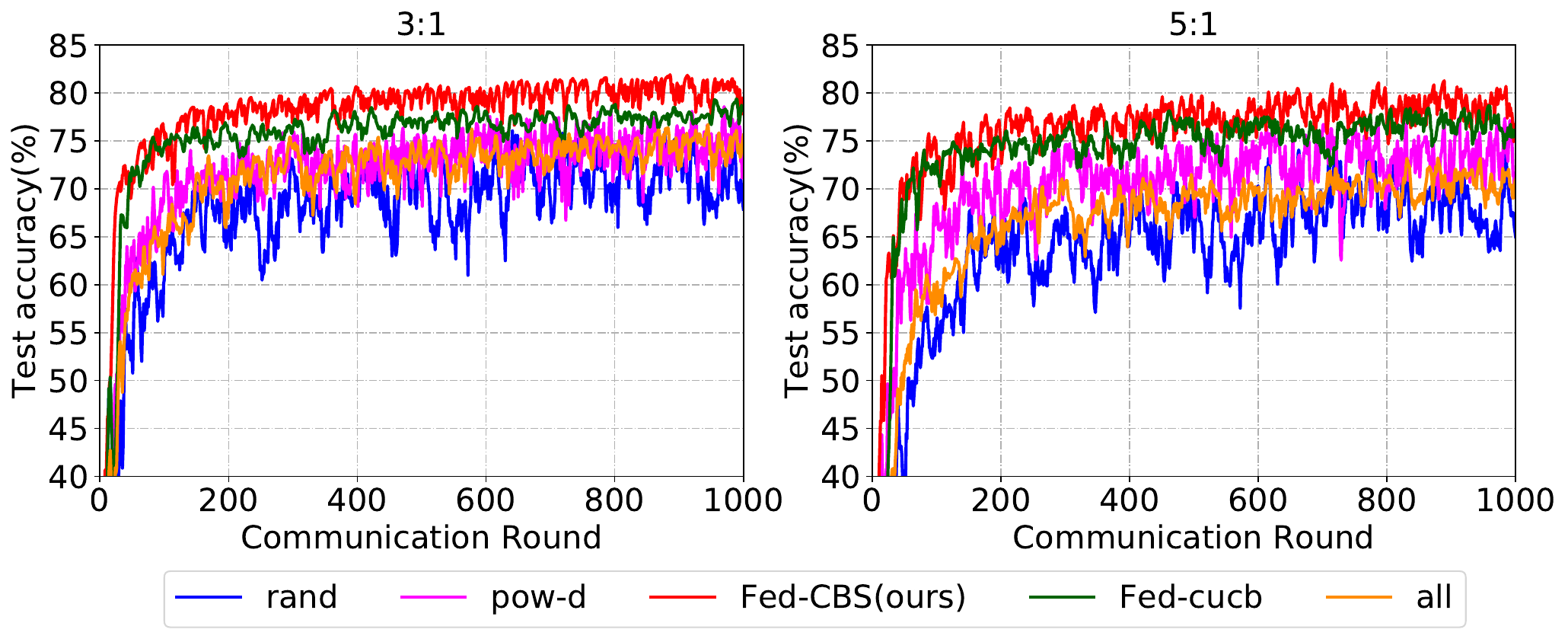}
  \caption{Case 1}
  \label{appfig:case1}
\end{subfigure}
\begin{subfigure}{.49\textwidth}
  \centering
  \includegraphics[width=\linewidth]{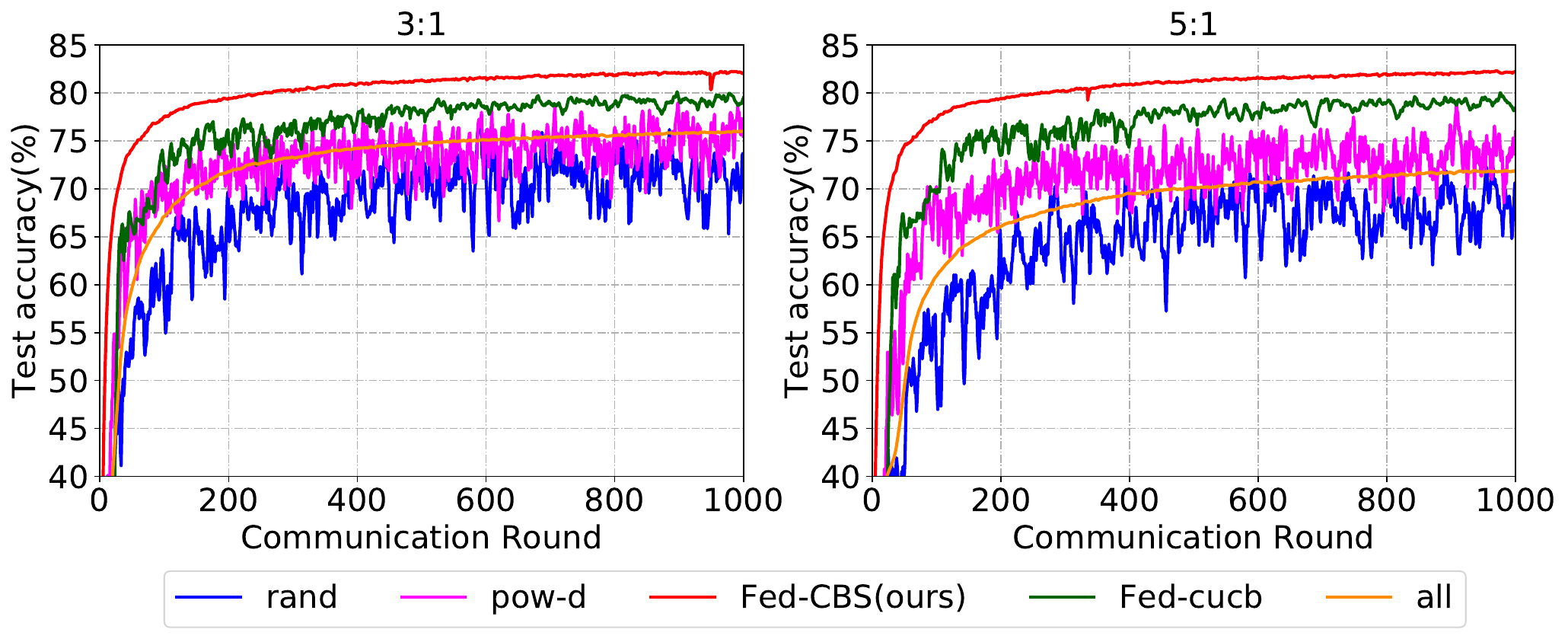}  
  \caption{Case 2}
  \label{appfig:case2}
\end{subfigure}
\caption{Test accuracy on Fashion-MINST with class-imbalanced global dataset in Case 1 and Case 2. }
\label{appfig:case12}
\end{figure*}

\begin{table*}[ht]
\centering
\renewcommand{\arraystretch}{1}
\setlength\tabcolsep{13pt}
\scalebox{1}{
\begin{tabular}{|c|c|c|c|c|c|c|}
\hline
\multicolumn{2}{|l|}{}                                                                          & all            & rand         & pow-d            & Fed-cucb                     & Fed-CBS                       \\ \hline
\multirow{3}{*}{Case 1}                                                            & 3:1 & 78.42$\pm$0.79 & 78.46$\pm$0.90 & 81.08$\pm$0.21 & 80.83$\pm$0.91                      &\textbf{81.75$\pm$0.34}          \\ \cline{2-7} 
                                                                                          & 5:1 & 72.42$\pm$2.22 & 75.49$\pm$2.56 & 80.15$\pm$0.41 &    80.50$\pm$0.95              & \textbf{81.42$\pm$0.50} \\ \hline
\multirow{2}{*}{Case 2}                                                            & 3:1 & 74.64$\pm$1.87 & 78.80$\pm$0.55 & 81.13$\pm$0.41 & 79.94$\pm$0.31 & \textbf{81.95$\pm$0.57}  \\ \cline{2-7} 
                                                                                          & 5:1 & 67.16$\pm$4.13 & 74.17$\pm$2.01 & 80.05$\pm$0.39 & 80.00$\pm$0.58 & \textbf{81.92$\pm$0.57}  \\ \hline
\end{tabular}}
\caption{Best test accuracy for our method and other four baselines on Fashion-MNIST dataset. }
\label{apptable:imbalancebest}
\end{table*}

Similar to the settings of Cifar-10, we assume that there are 200 clients in total. In each communication round, 30$\%$ of them (\textit{i.e.}, 60 clients) are set available in each training round. The global dataset of those 60 available clients is always class-imbalanced. To measure the degree of class imbalance, we make the global dataset have the same amount $n_1$ of data for the five classes and have the same amount $n_2$ of data for the other five classes. The ratio $r$ between $n_1$ and $n_2$ is set to $3:1$ and $5:1$. We select 10 clients from these 60 clients to participate in the training.

As shown in the Table \ref{apptable:imbalancebest} and Figure \ref{appfig:case2}, our method can achieve higher test accuracy and more stable convergence, which outperforms the ideal setting where all the available clients are engaged. Since the global dataset of the available 60 clients in each communication round is always class-imbalanced, the performance of engaging all of them is not good.

\section{Ablation Studies and Discussion}\label{Sec:ablation}

\subsection{Accurate Estimation vs Inaccurate Estimation for Fed-cucb}

\begin{figure*}[!hbtp]
  \centering
  \includegraphics[width=1\linewidth]{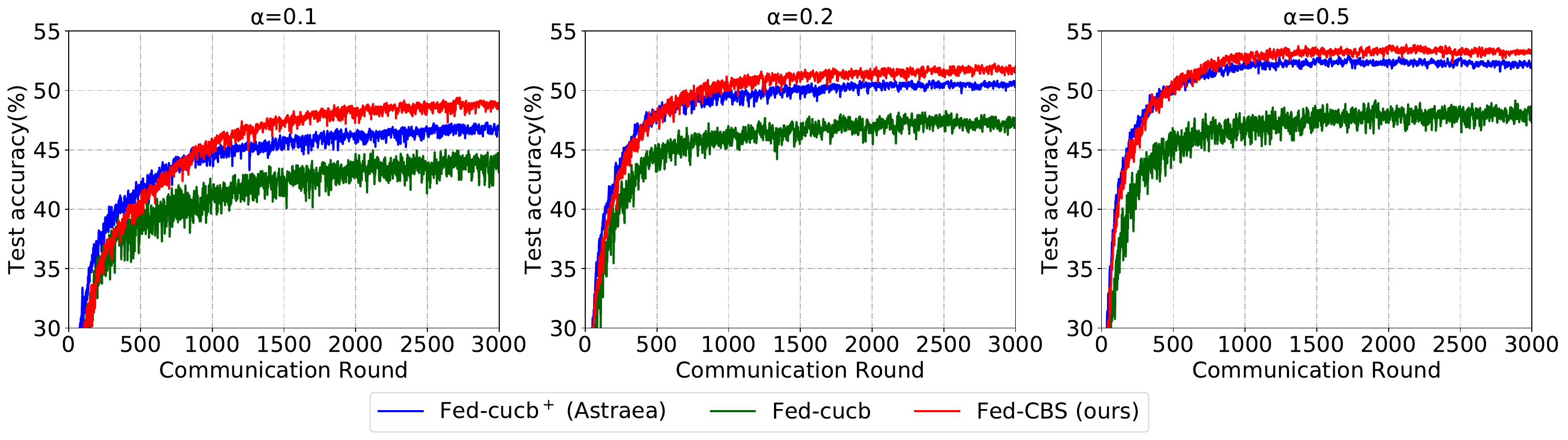}
\caption{Test accuracy on Cifar-10 for Fed-cucb, Fed-cucb$^+$ and Fed-CBS.}
\label{appfig:ablationfig1}
\end{figure*}

\begin{table*}[!hbtp]
\centering\renewcommand{\arraystretch}{1}
\setlength\tabcolsep{18pt}
\scalebox{1}{
\begin{tabular}{|c|c|c|c|c|}
\hline
\multicolumn{2}{|l|}{}                                                                                  & Fed-cucb$^+$ (Astraea) & Fed-cucb       & Fed-CBS           \\ \hline
\multirow{2}{*}{Best Accuracy ($\%$)}                                                           & $\alpha$=0.1 & 49.10$\pm$0.70 & 46.84$\pm$0.73 & \textbf{50.36$\pm$0.58} \\ \cline{2-5} 
                                                                                         & $\alpha$=0.2 & 50.61$\pm$0.77 & 48.80$\pm$1.05 & \textbf{51.95$\pm$0.57} \\ \cline{2-5} 
                                                                                         & $\alpha$=0.5 & 52.71$\pm$0.27 & 50.98$\pm$0.56 & \textbf{54.21$\pm$0.34} \\ \hline
\multirow{3}{*}{$\mathbb{E}(QCID)$ ({$10^{-2}$})} & $\alpha$=0.1 &0.83$\pm$0.18 &  7.09$\pm$2.27  & \textbf{0.62$\pm$0.20}  \\ \cline{2-5} 
                                                                                         & $\alpha$=0.2 & 0.68$\pm$0.05  &  5.93$\pm$1.01 & \textbf{0.51$\pm$0.12}  \\ \cline{2-5} 
                                                                                         & $\alpha$=0.5 & 0.43$\pm$0.04  & 6.47$\pm$0.77  & \textbf{0.36$\pm$0.04}  \\ \hline
\end{tabular}}
\caption{Best accuracy and the averaged $QCID$ values.}
\label{apptable:ablationtable1}
\end{table*}

As discussed in Sections \ref{sec:relatedwork} and \ref{sec:RCBGD}, the estimation of the label distribution in Fed-cucb \cite{yang2020federated} is not accurate, which leads to performance degradation. Hence there comes a natural question, would the performance of Fed-cucb get improved if it got an exact estimation of the local label distribution? In our simulation, we manually let the Fed-cucb know the exact value of each client's local label distribution and name it as Fed-cucb$^+$. Actually, Fed-cucb$^+$ is the core part of Astraea \cite{Duan2019AstraeaSF} without data augmentation. Hence, comparing our method with Fed-cucb$^+$ can show the superiority of our sampling strategy over the greedy method in Fed-cucb \cite{yang2020federated} and Astraea \cite{Duan2019AstraeaSF}.
\begin{figure*}[!hbtp]
  \centering
  \includegraphics[width=0.8\linewidth]{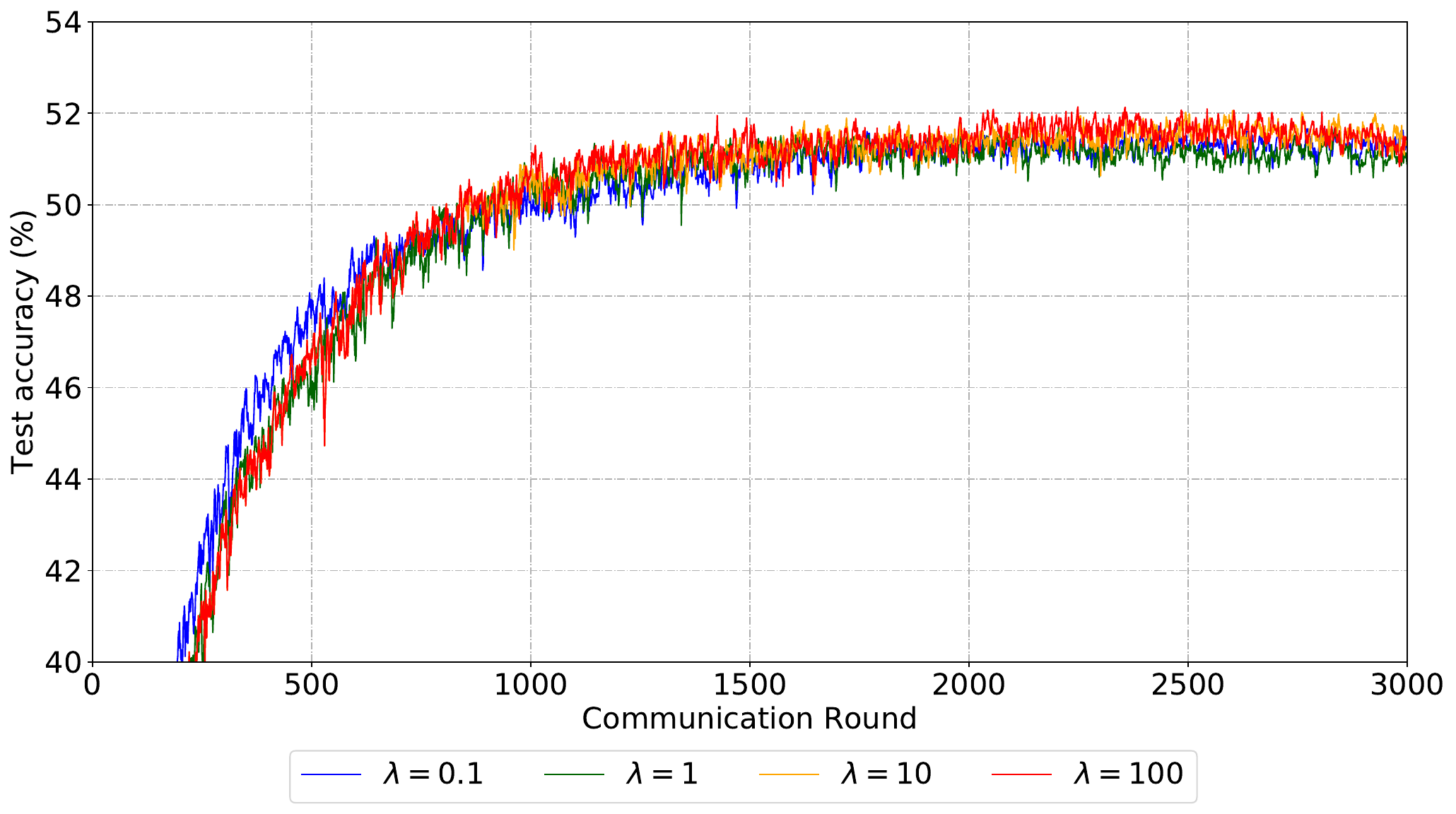}
\caption{Test accuracy with different exploration factor $\lambda$. }
\label{appfig:exploration}
\end{figure*}
\subsection{The Effect of Exploration Factor $\lambda$}

As shown in the Figure \ref{appfig:ablationfig1} and Table \ref{apptable:ablationtable1}, Fed-cucb$^+$ does improve the performance of Fed-cucb, which verifies the importance of accurate estimation. However, our Fed-CBS still outperforms Fed-cucb$^+$. Although, it seems that the accuracy of Fed-cucb$^+$ increases a little faster than Fed-CBS at the beginning of the training, our method will achieve higher accuracy as the training proceeds further. As discussed in the Remark \ref{remark1} in Section \ref{sec:samplingSt} and the Figure \ref{fig:greedyandours} of Section \ref{sec:RCBGD}, this is due to the pitfall of greedy method, where one will miss the optimal solution. This has been verified by the averaged $QCID$ value in Table \ref{apptable:ablationtable1}, which shows that Fed-CBS can achieve lower $\mathbb{E}(QCID)$ than Fed-cucb$^+$ (Astraea).

Another potential weakness of greedy method is the diversity of client composition. Following their selection process, once the first choice of client has been made, the following choices are fixed successively. Hence there are only limited kinds of client composition. It is interesting to investigate the relationship between the training performance and the diversity of client composition and we leave it as future work.

In our sampling strategy, when we sample the first client, we introduce the exploration factor $\lambda$ to balance the tradeoff between exploitation and exploration. When the $\lambda$ is small, our method will tend to exploit the class-balanced clients since their $QCID$ values are smaller. For fairness, we hope every client can get the chance to be selected. Hence, we can increase the $\lambda$ and then our method will tend to explore the clients which have seldom been selected before. However, it might cost many communication rounds for exploration and lead to slower convergence. 

We conduct some experiments to verify the effect of exploration factor $\lambda$. The settings are the same as the ones in Section \ref{sec:RCBGD} when $\alpha=0.2$. As shown in the Figure \ref{appfig:exploration}, as the $\lambda$ becomes larger, the increase of accuracy will become a little slower at the start of the training. This because the it might cost more communication rounds for exploration. As the training proceeds, the  accuracy with larger $\lambda$ becomes a little higher than the ones with smaller $\lambda$. Overall, the improvement on the convergence speed and best accuracy is very slight, which means the performance of FL training is not very sensitive to the values of exploration factor $\lambda$. Generally, if we want to slightly fasten the convergence, we can decrease the value of $\lambda$. If we want to improve the best accuracy a little, we can increase the value of $\lambda$.

\subsection{The Performance with Different Amounts of Selected Clients}
In this section, we want to investigate how the amount of selected clients will affect the FL training performance. Generally, we think as the amount of selected clients increases, the FL training process can achieve better performance. However, once that amount reaches some threshold $\epsilon$, the improvement will become slighter. This is because we find that select only a subset of all the available can achieve comparable results with engaging all the available clients into the training. As for how to decide the threshold $\epsilon$, we provide the following two principles based on $QCID$ and our experience.

\begin{itemize}
    \item First, if we work on a classification task with $B$ classes, we can select at least B clients. This is because in some special cases, each client will only have one class of data in their local dataset, such as the settings in Section \ref{RCIGD}. Hence, if less than B clients are selected, the grouped dataset of the selected clients will miss some classes of data.
    \item Second, to avoid missing some classes of data, we increase the threshold $\epsilon$ such that the averaged $QCID$ value could be smaller than $\frac{1}{B^2}$. This is because if the grouped dataset misses at least one class of data, the $QCID$ will be larger than $\frac{1}{B^2}$.
\end{itemize}

We conducted some experiments to verify our prediction on the effect of the amount of selected clients. The settings are the same as the ones in Section \ref{sec:RCBGD} when $\alpha=0.2$. As shown in the left figure of Figure \ref{appfig:amount}, as the amount of selected clients increases, the FL training process can achieve better performance. However, when the amount M is larger than 10, the improvement is slighter. In the right figure of Figure \ref{appfig:amount} , we can find that the averaged $QCID$ value of selecting 5 clients is larger than $(\frac{1}{10})^2=0.01$ and its performance is obviously worse than the others. These results verify the effectiveness of our principles on how to set the threshold $\epsilon$. It is worth noting that due to the limitation of communication capacities, we cannot select as many clients as possible. Hence, how to identify the appropriate threshold $\epsilon$ is critical to the FL training. 
\begin{figure*}[!hbtp]
  \centering
  \includegraphics[width=0.8\linewidth]{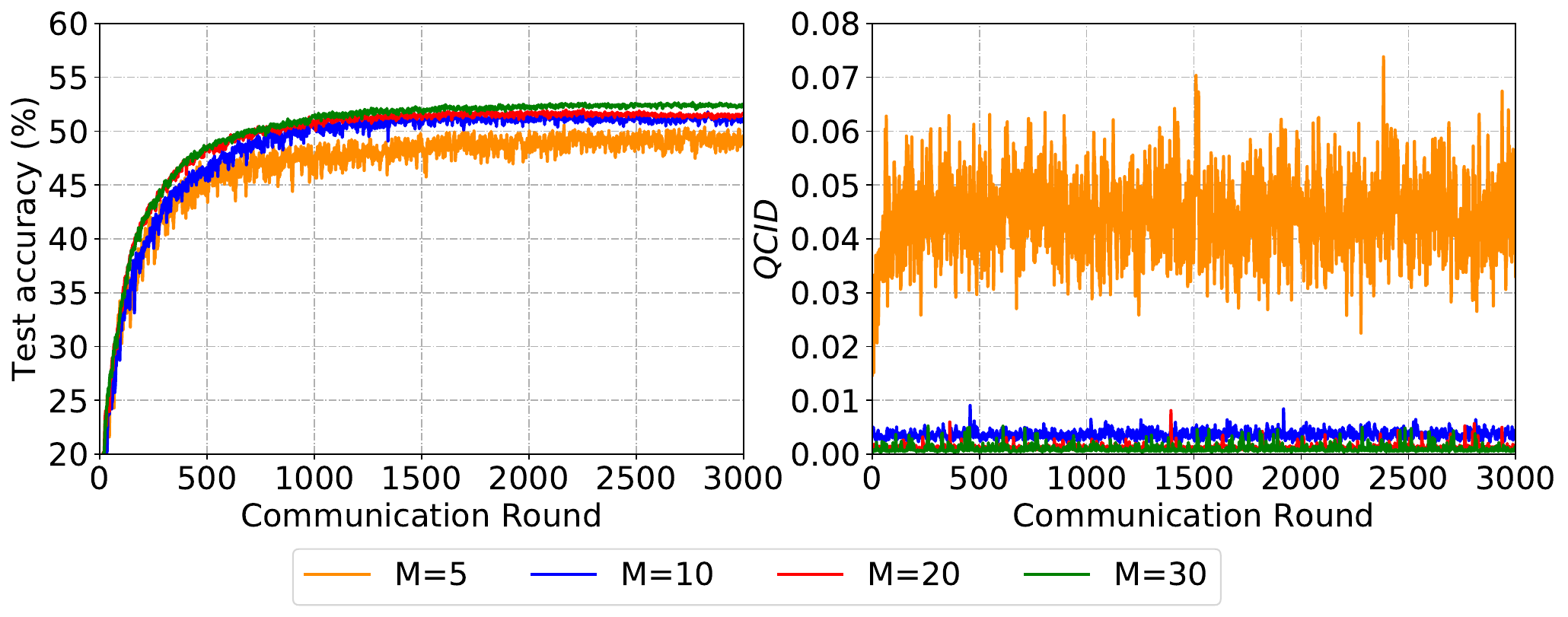}
\caption{Left: The performance with different amounts of selected clients. Right: The $QCID$ with different amounts of selected clients.}
\label{appfig:amount}
\end{figure*}

\subsection{Additional Experimental Results on FEMNIST Dataset}
We also conduct some experiments on the FEMINST Dataset to simulate more realistic settings where there are thousands of clients. Since in practice, it is impossible to engage all the clients during training, we compare our method by randomly selecting more clients. . There are 3500 ($>$ 1000) clients in total and we randomly set $10\% (< 30\%)$ of them available in each round. Then our method tries to select 30 clients from them. That is less than 1\% of all the 3550 clients and also less than the number of classes (64). Besides, we also run three baselines, randomly selecting 30 clients, randomly selecting $120(>100)$ clients, and selecting 30 clients with fed-cucb. We present the results in Table \ref{tablefemnist} and Figure \ref{figurefemnist}. Our performance is still the best. Due to the global imbalance, the rand-120 is even worse than the rand-30.

\begin{table}[ht]
\setlength\tabcolsep{6pt}
\centering
\begin{tabular}{|l|l|l|l|l|}
\hline
       & rand-30       & rand-120      & Fed-cucb     & Fed-CBS    \\ \hline
Communication Rounds & 1106 $\pm$ 24 & 1394 $\pm$ 11 & 1124$\pm$ 31 & \textbf{980$\pm$17} \\ \hline
\end{tabular}
\vspace{10pt}
\caption{The communication rounds required for targeted test accuracy (75\%). The results are the mean and the standard deviation over 3 different random seeds.}
\label{tablefemnist}
\end{table}

\begin{figure}[ht]
  \centering
  \includegraphics[width=0.95\linewidth]{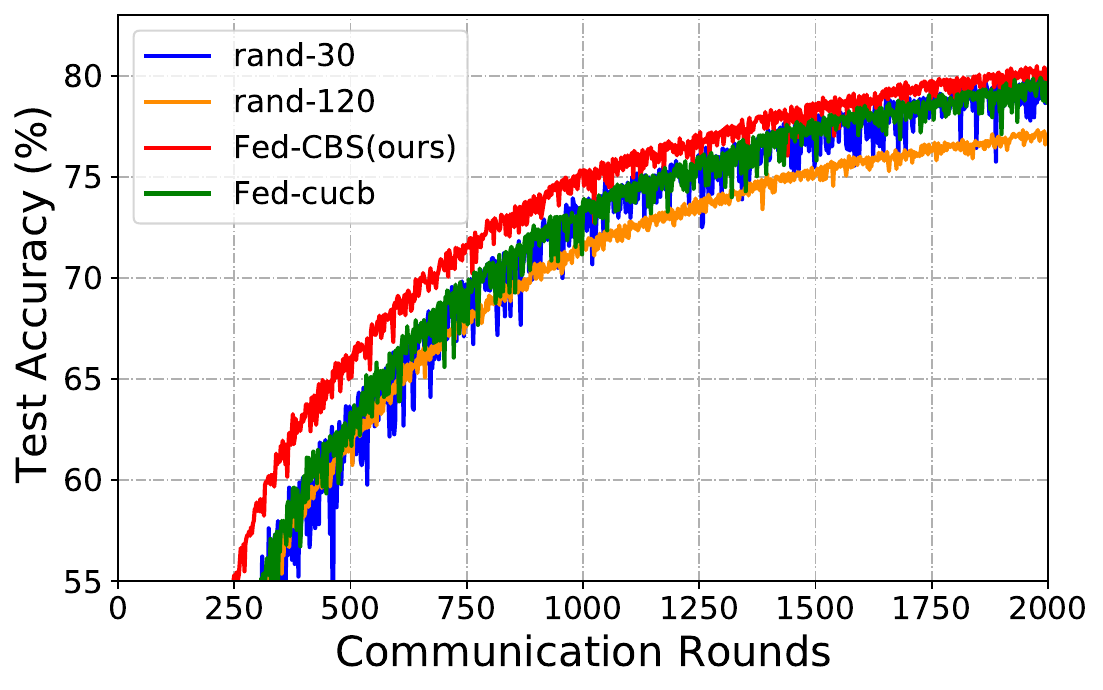}\\
\caption{Test accuracy for FEMNIST}\vspace{-10pt}
\label{figurefemnist}
\end{figure}


\section{Comparison between Cluster-based Client Sampling Algorithms and Fed-CBS}
We present the following comparison between cluster-based client sampling algorithms and our own method to demonstrate our superiority.

Firstly, the unbiased sampling property of the clustering sampling method \cite{fraboni2021clustered} may not lead to optimal performance when dealing with class-imbalanced global training datasets. In Section 3.1 of \cite{fraboni2021clustered}, the authors mention that they “require clustered sampling to be unbiased,” which implies that the expected value of client aggregation should be equivalent to the aggregation of all clients. However, our findings, as depicted in Figures \ref{fig1}, indicate that aggregating all clients does not always lead to satisfactory performance, especially when the downstream test task is class-balanced. It should be noted that ensuring class-balance in the downstream test task is crucial for maintaining fairness and privacy. This is because the imbalanced performance of the model across different classes can potentially reveal sensitive information about the global training dataset.

Secondly, our method guides the clustering sampling methods. Although clustering sampling can address many root causes of heterogeneity in the input space distributions at clients, however, since “unbiased sampling” will cause the mismatch between the input space distributions at clients and the downstream task's space distribution, we still need to identify key causes to make the clustering sampling “biased” to align the input space and downstream space. This is still very challenging because while clustering sampling methods can include many root causes of heterogeneity in the input space distributions at clients, we still need to be careful since most are hard to measure and contain lots of private information. Our analysis of "class imbalance" provides a valuable measure in this regard, and we also offer an efficient means of utilizing this measure in a privacy-preserving way. Therefore, our work can contribute to advancing clustering sampling methods in the future

\end{document}

%% file: math_commands.tex
\usepackage{amsmath,amsfonts,bm}

\def\eqref#1{equation~\ref{#1}}

\def\1{\bm{1}}

\def\mA{{\bm{A}}}

\def\mI{{\bm{I}}}

\def\mQ{{\bm{Q}}}

\def\mS{{\bm{S}}}

\DeclareMathAlphabet{\mathsfit}{\encodingdefault}{\sfdefault}{m}{sl}
\SetMathAlphabet{\mathsfit}{bold}{\encodingdefault}{\sfdefault}{bx}{n}













%% file: main.bbl
\begin{thebibliography}{45}
\providecommand{\natexlab}[1]{#1}
\providecommand{\url}[1]{\texttt{#1}}
\expandafter\ifx\csname urlstyle\endcsname\relax
  \providecommand{\doi}[1]{doi: #1}\else
  \providecommand{\doi}{doi: \begingroup \urlstyle{rm}\Url}\fi

\bibitem[Anand et~al.(1993)Anand, Mehrotra, Mohan, and
  Ranka]{anand1993improved}
Anand, R., Mehrotra, K.~G., Mohan, C.~K., and Ranka, S.
\newblock An improved algorithm for neural network classification of imbalanced
  training sets.
\newblock \emph{IEEE Transactions on Neural Networks}, 4\penalty0 (6):\penalty0
  962--969, 1993.

\bibitem[Anati et~al.(2013)Anati, Gueron, Johnson, and
  Scarlata]{anati2013innovative}
Anati, I., Gueron, S., Johnson, S., and Scarlata, V.
\newblock Innovative technology for cpu based attestation and sealing.
\newblock In \emph{Proceedings of the 2nd international workshop on hardware
  and architectural support for security and privacy}, volume~13, pp.\ ~7.
  Citeseer, 2013.

\bibitem[Balakrishnan et~al.(2021)Balakrishnan, Li, Zhou, Himayat, Smith, and
  Bilmes]{balakrishnan2021}
Balakrishnan, R., Li, T., Zhou, T., Himayat, N., Smith, V., and Bilmes, J.
\newblock Diverse client selection for federated learning: Submodularity and
  convergence analysis.
\newblock In \emph{ICML 2021 International Workshop on Federated Learning for
  User Privacy and Data Confidentiality}, Virtual, July 2021.

\bibitem[Brakerski et~al.(2014)Brakerski, Gentry, and Vaikuntanathan]{bgvfhe}
Brakerski, Z., Gentry, C., and Vaikuntanathan, V.
\newblock (leveled) fully homomorphic encryption without bootstrapping.
\newblock \emph{ACM Transactions on Computation Theory (TOCT)}, 6\penalty0
  (3):\penalty0 1--36, 2014.

\bibitem[Caldas et~al.(2018)Caldas, Wu, Li, Kone{\v{c}}n{\`y}, McMahan, Smith,
  and Talwalkar]{caldas2018leaf}
Caldas, S., Wu, P., Li, T., Kone{\v{c}}n{\`y}, J., McMahan, H.~B., Smith, V.,
  and Talwalkar, A.
\newblock Leaf: A benchmark for federated settings.
\newblock \emph{arXiv preprint arXiv:1812.01097}, 2018.
\newblock URL \url{https://arxiv.org/abs/1812.01097}.

\bibitem[Chatterji et~al.(2018)Chatterji, Flammarion, Ma, Bartlett, and
  Jordan]{chatterji2018theory}
Chatterji, N.~S., Flammarion, N., Ma, Y.-A., Bartlett, P.~L., and Jordan, M.~I.
\newblock On the theory of variance reduction for stochastic gradient monte
  carlo, 2018.

\bibitem[Chen et~al.(2017)Chen, Laine, and Player]{seal}
Chen, H., Laine, K., and Player, R.
\newblock Simple encrypted arithmetic library-seal v2.1.
\newblock In \emph{International Conference on Financial Cryptography and Data
  Security}, pp.\  3--18. Springer, 2017.

\bibitem[Chen et~al.(2013)Chen, Wang, and Yuan]{pmlr-v28-chen13a}
Chen, W., Wang, Y., and Yuan, Y.
\newblock Combinatorial multi-armed bandit: General framework and applications.
\newblock In Dasgupta, S. and McAllester, D. (eds.), \emph{Proceedings of the
  30th International Conference on Machine Learning}, volume~28 of
  \emph{Proceedings of Machine Learning Research}, pp.\  151--159, Atlanta,
  Georgia, USA, 17--19 Jun 2013. PMLR.
\newblock URL \url{https://proceedings.mlr.press/v28/chen13a.html}.

\bibitem[Chen et~al.(2020)Chen, Bhardwaj, and Marculescu]{chen2020fedmax}
Chen, W., Bhardwaj, K., and Marculescu, R.
\newblock Fedmax: mitigating activation divergence for accurate and
  communication-efficient federated learning.
\newblock \emph{arXiv preprint arXiv:2004.03657}, 2020.

\bibitem[Cho et~al.(2020)Cho, Wang, and Joshi]{Cho2020ClientSI}
Cho, Y.~J., Wang, J., and Joshi, G.
\newblock Client selection in federated learning: Convergence analysis and
  power-of-choice selection strategies.
\newblock \emph{ArXiv}, abs/2010.01243, 2020.

\bibitem[Defazio et~al.(2014)Defazio, Bach, and
  Lacoste-Julien]{defazio2014saga}
Defazio, A., Bach, F., and Lacoste-Julien, S.
\newblock Saga: A fast incremental gradient method with support for
  non-strongly convex composite objectives.
\newblock \emph{Advances in neural information processing systems}, 27, 2014.

\bibitem[Duan et~al.(2019)Duan, Liu, Chen, Tan, Ren, Qiao, and
  Liang]{Duan2019AstraeaSF}
Duan, M., Liu, D., Chen, X., Tan, Y., Ren, J., Qiao, L., and Liang, L.
\newblock Astraea: Self-balancing federated learning for improving
  classification accuracy of mobile deep learning applications.
\newblock In \emph{2019 IEEE 37th international conference on computer design
  (ICCD)}, pp.\  246--254. IEEE, 2019.

\bibitem[Fan \& Vercauteren(2012)Fan and Vercauteren]{fvfhe}
Fan, J. and Vercauteren, F.
\newblock Somewhat practical fully homomorphic encryption.
\newblock \emph{IACR Cryptol. ePrint Arch.}, 2012:\penalty0 144, 2012.

\bibitem[Fraboni et~al.(2021)Fraboni, Vidal, Kameni, and
  Lorenzi]{fraboni2021clustered}
Fraboni, Y., Vidal, R., Kameni, L., and Lorenzi, M.
\newblock Clustered sampling: Low-variance and improved representativity for
  clients selection in federated learning.
\newblock In \emph{International Conference on Machine Learning}, pp.\
  3407--3416. PMLR, 2021.

\bibitem[Goetz et~al.(2019)Goetz, Malik, Bui, Moon, Liu, and
  Kumar]{goetz2019active}
Goetz, J., Malik, K., Bui, D., Moon, S., Liu, H., and Kumar, A.
\newblock Active federated learning.
\newblock \emph{arXiv preprint arXiv:1909.12641}, 2019.

\bibitem[Halevi \& Shoup(2014)Halevi and Shoup]{helib}
Halevi, S. and Shoup, V.
\newblock Algorithms in {HElib}.
\newblock In \emph{Annual Cryptology Conference}, pp.\  554--571. Springer,
  2014.

\bibitem[Halevi \& Shoup(2015)Halevi and Shoup]{halevi2015bootstrapping}
Halevi, S. and Shoup, V.
\newblock Bootstrapping for {HElib}.
\newblock In \emph{Annual International conference on the theory and
  applications of cryptographic techniques}, pp.\  641--670. Springer, 2015.

\bibitem[Hao et~al.(2021)Hao, El-Khamy, Lee, Zhang, Liang, Chen, and
  Duke]{Hao_2021_CVPR}
Hao, W., El-Khamy, M., Lee, J., Zhang, J., Liang, K.~J., Chen, C., and Duke,
  L.~C.
\newblock Towards fair federated learning with zero-shot data augmentation.
\newblock In \emph{Proceedings of the IEEE/CVF Conference on Computer Vision
  and Pattern Recognition (CVPR) Workshops}, pp.\  3310--3319, June 2021.

\bibitem[Hsu et~al.(2019)Hsu, Qi, and Brown]{Hsu2019MeasuringTE}
Hsu, T.-M.~H., Qi, H., and Brown, M.
\newblock Measuring the effects of non-identical data distribution for
  federated visual classification.
\newblock \emph{ArXiv}, abs/1909.06335, 2019.

\bibitem[Huang et~al.(2016)Huang, Li, Loy, and Tang]{huang2016learning}
Huang, C., Li, Y., Loy, C.~C., and Tang, X.
\newblock Learning deep representation for imbalanced classification.
\newblock In \emph{Proceedings of the IEEE conference on computer vision and
  pattern recognition}, pp.\  5375--5384, 2016.

\bibitem[Huang et~al.(2021)Huang, Lin, Wu, He, Li, and Zomaya]{efficiencyBCSS}
Huang, T., Lin, W., Wu, W., He, L., Li, K., and Zomaya, A.~Y.
\newblock An efficiency-boosting client selection scheme for federated learning
  with fairness guarantee.
\newblock \emph{IEEE Transactions on Parallel \& Distributed Systems},
  32\penalty0 (07):\penalty0 1552--1564, jul 2021.
\newblock ISSN 1558-2183.
\newblock \doi{10.1109/TPDS.2020.3040887}.

\bibitem[Johnson \& Zhang(2013)Johnson and Zhang]{johnson2013accelerating}
Johnson, R. and Zhang, T.
\newblock Accelerating stochastic gradient descent using predictive variance
  reduction.
\newblock \emph{Advances in neural information processing systems}, 26, 2013.

\bibitem[Karimireddy et~al.(2019)Karimireddy, Kale, Mohri, Reddi, Stich, and
  Suresh]{karimireddy2019scaffold}
Karimireddy, S.~P., Kale, S., Mohri, M., Reddi, S.~J., Stich, S.~U., and
  Suresh, A.~T.
\newblock Scaffold: Stochastic controlled averaging for on-device federated
  learning.
\newblock 2019.

\bibitem[Krizhevsky et~al.()Krizhevsky, Nair, and Hinton]{cifar10}
Krizhevsky, A., Nair, V., and Hinton, G.
\newblock Cifar-10 (canadian institute for advanced research).
\newblock URL \url{http://www.cs.toronto.edu/~kriz/cifar.html}.

\bibitem[Li et~al.(2018)Li, Sahu, Zaheer, Sanjabi, Talwalkar, and
  Smith]{li2018federated}
Li, T., Sahu, A.~K., Zaheer, M., Sanjabi, M., Talwalkar, A., and Smith, V.
\newblock Federated optimization in heterogeneous networks.
\newblock \emph{arXiv preprint arXiv:1812.06127}, 2018.

\bibitem[Li et~al.(2019)Li, Huang, Yang, Wang, and Zhang]{li2019convergence}
Li, X., Huang, K., Yang, W., Wang, S., and Zhang, Z.
\newblock On the convergence of fedavg on non-iid data.
\newblock \emph{arXiv preprint arXiv:1907.02189}, 2019.

\bibitem[Liu \& Wang(2019)Liu and Wang]{liu2019stein}
Liu, Q. and Wang, D.
\newblock Stein variational gradient descent: A general purpose bayesian
  inference algorithm, 2019.

\bibitem[McMahan et~al.(2017{\natexlab{a}})McMahan, Moore, Ramage, Hampson, and
  y~Arcas]{mcmahan2017communication}
McMahan, B., Moore, E., Ramage, D., Hampson, S., and y~Arcas, B.~A.
\newblock Communication-efficient learning of deep networks from decentralized
  data.
\newblock In \emph{Artificial intelligence and statistics}, pp.\  1273--1282.
  PMLR, 2017{\natexlab{a}}.

\bibitem[McMahan et~al.(2017{\natexlab{b}})McMahan, Moore, Ramage, Hampson,
  et~al.]{mcmahan2016communication}
McMahan, H.~B., Moore, E., Ramage, D., Hampson, S., et~al.
\newblock {Communication-efficient Learning of Deep Networks from Decentralized
  Data}.
\newblock \emph{Artificial Intelligence and Statistics}, 2017{\natexlab{b}}.

\bibitem[Nishio \& Yonetani(2019)Nishio and Yonetani]{Nishio2019ClientSF}
Nishio, T. and Yonetani, R.
\newblock Client selection for federated learning with heterogeneous resources
  in mobile edge.
\newblock \emph{ICC 2019 - 2019 IEEE International Conference on Communications
  (ICC)}, pp.\  1--7, 2019.

\bibitem[{Reddi} et~al.(2020){Reddi}, {Charles}, {Zaheer}, {Garrett}, {Rush},
  {Kone{\v{c}}n{\'y}}, {Kumar}, and {McMahan}]{2020arXiv200300295R}
{Reddi}, S., {Charles}, Z., {Zaheer}, M., {Garrett}, Z., {Rush}, K.,
  {Kone{\v{c}}n{\'y}}, J., {Kumar}, S., and {McMahan}, H.~B.
\newblock {Adaptive Federated Optimization}.
\newblock \emph{arXiv e-prints}, art. arXiv:2003.00295, February 2020.

\bibitem[Ribero \& Vikalo(2020)Ribero and Vikalo]{ribero2020communication}
Ribero, M. and Vikalo, H.
\newblock Communication-efficient federated learning via optimal client
  sampling.
\newblock \emph{arXiv preprint arXiv:2007.15197}, 2020.

\bibitem[Schneider \& Barker(1989)Schneider and Barker]{schneider1989matrices}
Schneider, H. and Barker, G.
\newblock \emph{Matrices and Linear Algebra}.
\newblock Dover Books on Mathematics Series. Dover Publications, 1989.
\newblock ISBN 9780486660141.
\newblock URL \url{https://books.google.com/books?id=HJIT3CSb0wIC}.

\bibitem[Shen et~al.(2022)Shen, Cervino, Hassani, and
  Ribeiro]{shen2022agnostic}
Shen, Z., Cervino, J., Hassani, H., and Ribeiro, A.
\newblock An agnostic approach to federated learning with class imbalance.
\newblock In \emph{International Conference on Learning Representations}, 2022.

\bibitem[Wang et~al.(2020{\natexlab{a}})Wang, Liu, Liang, Joshi, and
  Poor]{wang2020tackling}
Wang, J., Liu, Q., Liang, H., Joshi, G., and Poor, H.~V.
\newblock Tackling the objective inconsistency problem in heterogeneous
  federated optimization.
\newblock \emph{arXiv preprint arXiv:2007.07481}, 2020{\natexlab{a}}.

\bibitem[Wang et~al.(2020{\natexlab{b}})Wang, Xu, Wang, and
  Zhu]{wang2020addressing}
Wang, L., Xu, S., Wang, X., and Zhu, Q.
\newblock Addressing class imbalance in federated learning.
\newblock \emph{arXiv preprint arXiv:2008.06217}, 2020{\natexlab{b}}.

\bibitem[Wang et~al.(2021)Wang, Xu, Wang, and Zhu]{wang2021addressing}
Wang, L., Xu, S., Wang, X., and Zhu, Q.
\newblock Addressing class imbalance in federated learning.
\newblock In \emph{Proceedings of the AAAI Conference on Artificial
  Intelligence}, volume~35, pp.\  10165--10173, 2021.

\bibitem[Welling \& Teh(2011)Welling and Teh]{welling2011bayesian}
Welling, M. and Teh, Y.~W.
\newblock Bayesian learning via stochastic gradient langevin dynamics.
\newblock In \emph{Proceedings of the 28th international conference on machine
  learning (ICML-11)}, pp.\  681--688, 2011.

\bibitem[Xiao et~al.(2017)Xiao, Rasul, and Vollgraf]{xiao2017fashionmnist}
Xiao, H., Rasul, K., and Vollgraf, R.
\newblock Fashion-mnist: a novel image dataset for benchmarking machine
  learning algorithms, 2017.
\newblock URL \url{http://arxiv.org/abs/1708.07747}.
\newblock cite arxiv:1708.07747Comment: Dataset is freely available at
  https://github.com/zalandoresearch/fashion-mnist Benchmark is available at
  http://fashion-mnist.s3-website.eu-central-1.amazonaws.com/.

\bibitem[Yang et~al.(2020)Yang, Wong, Zhu, Wang, and Qian]{yang2020federated}
Yang, M., Wong, A., Zhu, H., Wang, H., and Qian, H.
\newblock Federated learning with class imbalance reduction, 2020.

\bibitem[Yang et~al.(2021)Yang, Zhang, Hao, Spell, and Carin]{yang2021flop}
Yang, Q., Zhang, J., Hao, W., Spell, G.~P., and Carin, L.
\newblock Flop: Federated learning on medical datasets using partial networks.
\newblock In \emph{Proceedings of the 27th ACM SIGKDD Conference on Knowledge
  Discovery \& Data Mining}, pp.\  3845--3853, 2021.

\bibitem[Zhang et~al.(2020{\natexlab{a}})Zhang, Zhang, Carin, and
  Chen]{pmlr-v108-zhang20d}
Zhang, J., Zhang, R., Carin, L., and Chen, C.
\newblock Stochastic particle-optimization sampling and the non-asymptotic
  convergence theory.
\newblock In Chiappa, S. and Calandra, R. (eds.), \emph{Proceedings of the
  Twenty Third International Conference on Artificial Intelligence and
  Statistics}, volume 108 of \emph{Proceedings of Machine Learning Research},
  pp.\  1877--1887. PMLR, 26--28 Aug 2020{\natexlab{a}}.
\newblock URL \url{https://proceedings.mlr.press/v108/zhang20d.html}.

\bibitem[Zhang et~al.(2020{\natexlab{b}})Zhang, Zhao, and
  Chen]{zhang2020variance}
Zhang, J., Zhao, Y., and Chen, C.
\newblock Variance reduction in stochastic particle-optimization sampling.
\newblock In \emph{International Conference on Machine Learning}, pp.\
  11307--11316. PMLR, 2020{\natexlab{b}}.

\bibitem[Zhang et~al.(2019)Zhang, Li, Zhang, Chen, and
  Wilson]{zhang2019cyclical}
Zhang, R., Li, C., Zhang, J., Chen, C., and Wilson, A.~G.
\newblock Cyclical stochastic gradient mcmc for bayesian deep learning.
\newblock \emph{arXiv preprint arXiv:1902.03932}, 2019.

\bibitem[Zhao et~al.(2018)Zhao, Zhang, and Chen]{zhao2018selfadversarially}
Zhao, Y., Zhang, J., and Chen, C.
\newblock Self-adversarially learned bayesian sampling, 2018.

\end{thebibliography}
